\documentclass{article}
\usepackage{fullpage}
\usepackage{amssymb}
\usepackage{amsmath}
\usepackage{amsthm}
\usepackage{multirow}
\usepackage{enumerate}
\usepackage{graphicx, subcaption}
\usepackage{mathrsfs}
\usepackage[utf8]{inputenc}
\usepackage{cite}
\usepackage{hyperref}
\usepackage{cleveref}
\usepackage{mathtools}
\usepackage{amsfonts}
\usepackage[T1]{fontenc}
\usepackage{mathpazo}
\usepackage{bm}
\usepackage{isomath}
\usepackage{verbatim}
\usepackage{changepage}
\usepackage[usenames,dvipsnames,svgnames,table]{xcolor}

\definecolor{blueviolet}{RGB}{60,50,200}
\definecolor{oliveg}{RGB}{40,200,30}
\hypersetup{colorlinks=true,       
    linkcolor=blueviolet,
    citecolor=oliveg,
}

\usepackage{tikz}
\usetikzlibrary{calc}
\usepackage[ruled,vlined,linesnumbered]{algorithm2e}
\usepackage{color}
\usepackage{todonotes}

\theoremstyle{definition}
\newtheorem{definition}{Definition}[section]

\theoremstyle{plain}
\newtheorem{lemma}{Lemma}[section]
\newtheorem{theorem}{Theorem}

\newtheorem{example}{Example}[section]
\newtheorem{question}{Question}[section]
\newtheorem{assumption}{Assumption}[section]
\newtheorem{remark}{Remark}[section]

\makeatletter
\newcommand\footnoteref[1]{\protected@xdef\@thefnmark{\ref{#1}}\@footnotemark}
\makeatother

\title{Disentangling Mixtures of Unknown Causal Interventions\footnote{Published as an oral paper  \cite{Kumar2021} at the $37^{th}$ Conference on Uncertainty in Artificial Intelligence (UAI), 2021.}}

\author{Abhinav Kumar\\
\normalsize{Microsoft Research, Bangalore}\\
\normalsize{Karnataka, India}\\
\normalsize{\tt abhinavkumar.wk@gmail.com}

\and {Gaurav Sinha}\\
\normalsize{Microsoft Research, Bangalore}\\
\normalsize{Karnataka, India}\\
\normalsize{\tt sinhagaur88@gmail.com}}

\date{}

\begin{document}
\maketitle

\begin{abstract}
In many real-world scenarios, such as gene knockout experiments, targeted interventions are often accompanied by unknown interventions at off-target sites. Moreover, different units can get randomly exposed to different unknown interventions, thereby creating a mixture of interventions. Identifying different components of this mixture can be very valuable in some applications. Motivated by such situations, in this work, we study the problem of identifying all components present in a mixture of interventions on a given causal Bayesian Network. We construct an example to show that, in general, the components are not identifiable from the mixture distribution. Next, assuming that the given network satisfies a \emph{positivity} condition, we show that, if the set of mixture components satisfy a mild \emph{exclusion} assumption, then they can be uniquely identified. Our proof gives an efficient algorithm to recover these targets from the exponentially large search space of possible targets. In the more realistic scenario, where distributions are given via finitely many samples, we conduct a simulation study to analyze the performance of an algorithm derived from our identifiability proof. 
\end{abstract}

\section{Introduction}
\label{sec:intro}

\textbf{Motivation:} Causal Bayesian Networks (CBN) (\cite{Pearl2009, Spirtes2010}), have become the popular choice to model causal relationships in many real-world systems. These models can simulate the effects of external interventions that forcibly fix target system variables to desired target values. The simulation is done via the $do()$  operator (\cite{Pearl2009}) wherein the CBN is altered by breaking incoming edges of the target variables and fixing them to desired target values. Pre-estimating the effect of interventions can help in decision making, for example, interventions on a CBN describing gene interactions can guide gene editing experiments.

However, real-world interventions are not always precise and mistakenly end up intervening other unintended targets. For example, gene knockout experiments via the CRISPR-Cas$9$ gene-editing technology perform unintended cleavage at unknown genome sites (\cite{Fu2013, Wang2015}). Moreover, the unintended intervention targets can themselves be noisy i.e. different individuals targeted by the same intervention might undergo completely different off-target interventions.  For example, \cite{Aryal2018} demonstrated that same gene editing experiment (using CRISPR-Cas$9$) on mice embryos exhibited different unintended cleavage for different mice. 
In such situations, units (samples) that underwent different unintended interventions are not segregated and therefore the generated distribution becomes a mixture of individual interventional distributions. We ask the following natural question.



\begin{question}
\label{ques:main}
Given access to a mixture of interventional distributions, under what conditions can one identify all the intervention targets?\footnote{from here on wards we use the terms targets and components interchangeably}
\end{question}

\textbf{Our Contributions:}
First, we model the situation of identifying hidden off-target interventions as the problem of identifying individual components of a mixture of interventions. We assume an underlying CBN and model interventions via the $do()$ operator described above.  Second, by constructing examples, we show that, in general for a given CBN and an input mixture distribution, components of the mixture might not be unique. Using this, we motivate the need for a mild \emph{positivity} assumption (Assumption \ref{assm:positivity})  on the distribution generated by the CBN and a mild and reasonable \emph{exclusion} assumption (Assumption \ref{assm:intervention-exclusion}) on the structure of the intervention components present in the mixture. Third, we prove that, given access to a CBN satisfying \emph{positivity} and any input mixture having intervention components satisfying \emph{exclusion}, such intervention components generating the mixture can be uniquely identified from the mixture distribution. Fourth, given oracle access to marginals of the distributions generated by the CBN and the mixture, our identifiability proof gives an efficient algorithm to recover target components from an exponentially large space of possible components. Finally, in Section \ref{sec:exp}, we conduct a simulation study to analyze the performance of an algorithm (Algorithm \ref{algo:general-algo}) directly inspired from our identifiability proof, but with access to only finitely many samples. Even though the goal of our paper is to prove identifiability of these intervention targets, our simulation study indicates that our algorithm is promising in the realistic situation of finitely many samples.

\textbf{Related Prior Work:}
Recently \cite{Squires2020} considered the problem of causal discovery using unknown intervention targets, and, as a crucial intermediate step, prove identifiability of these targets. They also design two algorithms UT-IGSP and JCI-GSP (based on the Joint Causal Inference framework in \cite{Mooij2020}) to recover these targets from data. As discussed in our motivation, in many real situations, such as \cite{Aryal2018}, the off-target effects are themselves noisy and end up creating mixtures of multiple unknown interventions. Since \cite{Squires2020} assumes separate access to each unknown intervention, their algorithm cannot be used in our situation.
Another line of work related to ours is the study of mixtures of Bayesian Networks. Perfect interventions i.e. $do()$ operators on the CBNs create new interventional CBNs (Definition $1.3.1$ in \cite{Pearl2009}) and therefore the input mixture in our setup is actually a mixture of Bayesian Networks. This is a more general problem and was tackled first in \cite{Thiesson1998}. They developed an Expectation-Maximization (EM) based heuristic to find individual Bayesian Network components. However, they do not investigate identifiability of the components. In our setting, we care about identifiability since the components correspond to the unknown interventions. Along with recovering the individual components of a mixture, there is also growing interest in developing techniques to understand conditional independence (CI) relationships among the variables in the mixture data. For example, some recent works try to build other graphical representations, from which the CI relationships in the mixture can be easily understood (\cite{Spirtes1994, Ramsey2011, Strobl2019a, Strobl2019b, Saeed2020}). Even though these new representations can identify some aspects of the components, none of these works prove or discuss the uniqueness and identifiability of the components, which is the main interest of our work. Finally, we would like to mention that the general area of causal discovery and inference using different kinds of unknown interventions has received a lot of attention lately (\cite{Eaton2007, Squires2020, Jaber2020, Mooij2020, Rothenhausler2015}). Even though many of these do not align with goal of our paper, the growing interest in this area highlights seriousness of the issue of unintended stochasticity in targeted interventions and the desire to design algorithms robust to them.

\section{Preliminaries}
\label{sec:prelim}

\textbf{Notation} We use capital letters (e.g. $X$) to represent random variables and the corresponding lower case letter $x$ to denote the assignment $X=x$. The set of values taken by random variable $X$ will be denoted by $C_X$. Unless otherwise specified, all random variables in this paper are discrete and have finite support i.e. $|C_X|<\infty$. A tuple or set of random variables is denoted by capital bold face letter (e.g. $\bm{X}$) and the corresponding lower case bold faced letter $\bm{x}$ will denote the assignment $\bm{X} = \bm{x}$. Let, $C_{\bm{X}}=\prod_{X_i\in \bm{X}}C_{X_i}$ denote the set of all possible values that can be taken by $\bm{X}$. Probability of $\bm{X}$ taking the value $\bm{x}$ is denoted by $\mathbb{P}(\bm{X=x})$ or equivalently as $\mathbb{P}(\bm{x})$ and probability of $\bm{X=x}$ given $\bm{Y=y}$ is denoted as $\mathbb{P}(\bm{X=x|Y=y})$ or equivalently with $\mathbb{P}(\bm{x|y})$.  We will use $[n]$ to denote the set $\{1,2,...,n\}$, $[m,n]$ to denote set $\{m,m+1,\ldots,n\}$, calligraphic capital letters e.g. $\mathcal{S}$ to denote sets. Size of any set $\mathcal{S}$ is denoted by $|\mathcal{S}|$. $\mathbb{R}, \mathbb{R}_{+}$ and $\mathbb{R}_{\geq 0}$ will denote the set of real numbers, positive real numbers and non-negative real numbers respectively.

\textbf{Bayesian Network}
Let $\mathcal{G}=\{\bm{V},\mathcal{E}\}$ be a directed acyclic graph (DAG) with node set $\bm{V}=\{V_1,\ldots,V_n\}$ where each node $V_i$ represents a random variable. $\mathcal{G}$ is called a Bayesian Network if the following factorization of the joint probability of $\bm{V}$ holds.
\begin{equation*}
    \label{eq:graph-factor}
    \mathbb{P}(\bm{V}) = \prod_{V_i \in \bm{V}} \mathbb{P}(V_i|\bm{pa}(V_i))
\end{equation*}
where $\bm{pa}(V_i)$ are parent nodes of $V_i$.


A \emph{causal Bayesian Network} is a Bayesian Network where all edges denote direct causal relationships. It allows for modeling effect of external actions called ``interventions'', by appropriate modification of the Bayesian Network. A formal definition of causal Bayesian Networks can be found in Definition $1.3.1$, \cite{Pearl2009}.

\textbf{Interventions:} As mentioned above, these capture external actions on a system under consideration, for example, dosage of medicines administered to a patient, providing subsidies to poorer sections of the population, etc. A natural way to model them in causal Bayesian Networks is to perform the act of \emph{causal surgery}, wherein, incoming edges into the node(s) to be intervened are removed and the node(s) is forcibly fixed to the desired value.
As described in Definition $1.3.1$, \cite{Pearl2009}, the new network thus obtained is treated as the Bayesian Network modelling effect of the intervention. Formally, following the notation in \cite{Pearl2009}, if we perform intervention on nodes $\bm{X}\subseteq \bm{V}$ with a desire to set it to value $\bm{x}^{*}\in C_{\bm{X}}$, then the effect of this intervention (also known as \emph{interventional distribution}) is a probability distribution on $\bm{V}$ denoted as $\mathbb{P}(\bm{v}|do(\bm{x}^{*}))$ (or $\mathbb{P}_{\bm{x}^{*}}(\bm{v})$). In the intervened Bayesian Network, conditional probability distributions (CPD) $\mathbb{P}(X_i|\bm{pa}(X_i))$ of all $X_i\in \bm{X}$ that are intervened and set to $x_i^*$, changes to the Kronecker delta function $\delta_{x_i, x_i^*}$ i.e. $\mathbb{P}(X_i=x_i|\bm{pa}(X_i)) = 1$ if $x_i=x_i^*$ else it is $0$. The CPD of the non-intervened nodes i.e. $\bm{V}\setminus\bm{X}$ remains unchanged. Hence the interventional distribution factorizes as:
\begin{equation*}
\label{eq:interv-factor}
    \mathbb{P}_{\bm{x}^{*}}(V=\bm{v}) = \prod_{V_i \notin \bm{X}} \mathbb{P}(V=v_i|\bm{pa}(V_i)) \prod_{V_i \in \bm{X}} \delta_{v_i,x_i^{*}}
\end{equation*}

Such interventions are called \textit{perfect interventions}. They capture many real-world situations like \textit{gene-editing-experiments}, where a certain target gene is spliced out and replaced with the desired gene. Other kinds of interventions such as \emph{imperfect, uncertain} e.t.c. have been defined in literature (Section 2 in \cite{Eaton2007}). However, in this paper we only deal with perfect interventions.



\section{Problem Formulation and Main Theorem}
\label{sec:problem-formulation}
As motivated in Section \ref{sec:intro}, the intended interventions performed during an experiment often have hidden off-target effects, which could themselves be stochastic, leading to different hidden treatments on different individuals. We can model such a situation as an unknown mixture of different interventions. Here is a formal definition.

\begin{definition}[Mixture of Interventions]
\label{def:general-mix}
Let $\mathcal{G}=\{\bm{V},\mathcal{E}\}$ be a causal Bayesian Network. A probability distribution $\mathbb{P}_{mix}(\bm{V})$ is called a mixture of interventions if for some $m\in \mathbb{N}$, there exist subsets ${\bm T}_1, \ldots, {\bm T}_m \subseteq {\bm V}$, corresponding values $\bm{t}_i \in \mathcal{C}_{{\bm T}_i}$, and positive scalar weights $\pi_i \in \mathbb{R}_{+}$, $i\in [m]$, such that
\begin{equation*}
    \mathbb{P}_{mix}({\bm V}) = \sum\limits_{i=1}^m \pi_i \mathbb{P}_{{\bm t_i}}({\bm V})
\end{equation*}
where $\bm{t}_i\neq \bm{t}_j$ for all $i\neq j \in [m]$\footnote{if $t_i=t_j$, then $(\pi_i + \pi_j)\mathbb{P}_{{\bm t_i}}(\bm{V})$ is one component.}.
We allow $\bm{T_i} =\varnothing$, in which case, $\mathbb{P}_{{\bm t_i}}(\bm{V})$ is defined as $\mathbb{P}(\bm{V})$. Note that for $\mathbb{P}_{mix}$ to be a valid distribution $\sum_{i=1}^m \pi_i=1$.
We refer to the set $\mathcal{T} = \{(\bm{t}_i, \pi_i), i\in [m]\}$ as a set of \textit{intervention tuples} generating the mixture.

\end{definition}

\textbf{Uniqueness and Identifiability :} In our mixture model, each of the targets $\bm{t}_i$, corresponds to an intervention that intentionally or unintentionally transpired in the experiment. Since our ultimate goal is to recover them from the mixture distribution (see Question \ref{ques:main}), the problem only makes sense if they ``uniquely'' define the mixture. Formally, there should not exist two distinct sets of intervention tuples $\mathcal{T}_1=\{({\bm t}_1^1,\pi_1^1),\ldots, ({\bm t}_n^1,\pi_n^1) \}$ and  $\mathcal{T}_2=\{({\bm t}_1^2,\pi_1^2),\ldots, ({\bm t}_m^2,\pi_m^2) \}$ which generate the same mixture distribution, i.e.,
\begin{equation*}
    \mathbb{P}_{mix}({\bm V})=\sum\limits_{i=1}^{n}\pi_i^1 \mathbb{P}_{{\bm t}_i^1}({\bm V})=\sum\limits_{j=1}^{m}\pi_j^2 \mathbb{P}_{{\bm t}_j^2}({\bm V})
\end{equation*}

An immediate next question is that of ``identifiability''. 
Given access to a causal Bayesian Network and the joint distribution $\mathbb{P}(\bm{V})$ it captures, does there exist an algorithm, that takes as input the mixture distribution $\mathbb{P}_{mix}(\bm{V})$ and exactly recovers the unknown set of intervention tuples  that generated $\mathbb{P}_{mix}(\bm{V})$?
In the general case, the answer to both these questions is no! Using a very simple network, with just one node, we show that mixture distributions need not be unique, motivating the need for more assumptions. More complicated examples with multiple nodes can be easily created in the same way, but, for a cleaner presentation we stick to this example since its purpose is to only motivate an assumption we make next.

\begin{example} 
\label{eg:exclusion}
Consider a causal Bayesian Network with a single binary variable $\bm{V}=\{V_1\}$, i.e. $C_{V_1} = \{0,1\}$ and denote $\mathbb{P}(V_1=0), \mathbb{P}(V_1=1)$ by $p_0, p_1$ respectively. Define the mixture,
\begin{equation*}
    \mathbb{P}_{mix}(V_1) =  \pi_0 \mathbb{P}_{0}(V_1) + \pi_1 \mathbb{P}_1(V_1) + (1-\pi_0-\pi_1) \mathbb{P}(V_1)
\end{equation*}
On setting $V_1=0$ and then $V_1=1$ in the above equation, and rearranging the terms, we obtain
\begin{equation*}
\label{eq:example-exclusion-matrix}
    \begin{bmatrix}
        \begin{array}{cc}
             1-p_0 & -p_0 \\
            p_0-1 & p_0 \\
        \end{array}
    \end{bmatrix}
    \begin{bmatrix*}[l]
        \pi_0 \\ \pi_1
    \end{bmatrix*}
        =
    \begin{bmatrix*}[l]
        \begin{array}{cc}
             & \mathbb{P}_{mix}(V_1=0)-p_0 \\
             & \mathbb{P}_{mix}(V_1=1)-p_1 \\
        \end{array}
    \end{bmatrix*}
\end{equation*}

The above $2\times 2$ matrix is singular and has rank $1$ i.e. the system does not have a unique solution. In fact, when $0<p_0<1$, 
\[
\pi_0=\frac{\mathbb{P}_{mix}(V_1=0)-p_0 + p_0t}{1-p_0}, \hspace{1em} \pi_1=t
\]
are all valid solutions whenever $t \leq 1-\mathbb{P}_{mix}(V_1=0)$ and $t\geq \max\{\frac{p_0-\mathbb{P}_{mix}(V_1=0)}{p_0}, 0\}$. Therefore, uniqueness of intervention tuples does not hold in general.
\end{example}

Even though the example looks very simple, it captures the main reason behind the non-identifiability of the set of intervention tuples. Exactly like the above example, for any mixture, we can obtain systems of linear equations by evaluating marginal probabilities of $\mathbb{P}_{mix}$ for different settings of $\bm{V}$. Our goal then would be to find settings which help us solve these systems uniquely and recover the set of intervention tuples. Unfortunately, in this process, similar to the above example, the linear systems will have dependent equations and therefore infinitely many solutions. To get over this issue, we focus our attention on sets of intervention tuples, where, for each variable there exists some value that is missing from all of it's intervention targets. In, our main theorem, we show that any mixture generated by such a set cannot be generated by any other set of this kind. Next, we formally state the assumption and then discuss why it is extremely mild and reasonable in most real situations.

\begin{assumption}[Exclusion]
\label{assm:intervention-exclusion}
Let $\mathcal{T}$ be a set of intervention tuples as defined in Definition \ref{def:general-mix}. We say that $\mathcal{T}$ satisfies \emph{exclusion}, if for all $V_i\in \bm{V}$, there exists $\Bar{v}_i\in C_{V_i}$ such that $\Bar{v}_i\notin \bm{t}$ for any target $\bm{t}$ belonging to any tuple in $\mathcal{T}$. We say that a mixture of interventions $\mathbb{P}_{mix}(\bm{V})$ satisfies \emph{exclusion} if some set of intervention tuples $\mathcal{T}$ generating it satisfies exclusion.
\end{assumption}

\begin{remark}
This assumption puts only a mild constraint on the set of mixtures we consider. For example, in a network with $n$ nodes and each node having $\leq k$ possible values, excluding a fixed value of each node, can still generate arbitrary mixtures over $\Omega(k^n)$ allowed targets. Without exclusion, there are $O((k+1)^n)$ possible targets that generate the mixtures. Therefore the reduction is minimal compared to the size of the space of targets we are searching in. In real-world applications, it's common for nodes to have a large number of possible values. Therefore, for each node, the possibility of off-target interventions impacting all values becomes unlikely. We also emphasize that the values missing from the targets can be different for different input mixtures and are not known to our algorithms. Our identifiability algorithm only uses existence of such missing values making its interpretation even more general.
\end{remark}

Even though the above assumption helps us tackle the singularity problem outlined in Example \ref{eg:exclusion}, it is not enough to guarantee uniqueness of intervention tuples in general. We also assume a simple ``positivity'' assumption on the causal Bayesian Network, which demands that the joint probability $\mathbb{P}(\bm{v})>0$ for any setting $\bm{V = v}$. In fact, using the same example as above (Example \ref{eg:exclusion}), we show that not assuming $p_0, p_1 > 0$, can lead to multiple set of intervention tuples satisfying Assumption \ref{assm:intervention-exclusion} and generating the same mixture. To see this, we consider the input mixture $\mathbb{P}_{mix}(\bm{V}) = \mathbb{P}(\bm{V})$. The set of intervention tuples $\mathcal{T}_1 = \{(\varnothing, 1)\}$ for it clearly satisfies Assumption \ref{assm:intervention-exclusion} as intervention targets $(V_1=a)$ and $(V_1=b)$ are excluded. Now, if $p_1=0$, then $\mathbb{P}_0(V_1) = \mathbb{P}(V_1)$ and for any $\pi_0\in [0,1]$, we can trivially write
\[
    \mathbb{P}_{mix}(V_1) =  \pi_0 \mathbb{P}_{0}(V_1) + (1-\pi_0) \mathbb{P}(V_1)
\]
implying that $\mathcal{T}_2 = \{(V_1=0, \pi_0), (\varnothing, 1-\pi_0)\}$ is another set of intervention tuples for $\mathbb{P}_{mix}$, implying non-uniqueness. Here is the statement of our assumption.

\begin{assumption}[Positivity]
\label{assm:positivity}
Let $\bm{V}$ be the set of nodes in our causal Bayesian Network and $\mathbb{P}(\bm{V})$ be the corresponding joint probability distribution. We assume that $\mathbb{P}(\bm{v})>0$ for all $\bm{v}\in C_{\bm{V}}$.
\end{assumption}

\begin{remark}
As a straight forward consequence of this assumption, for every random variable $V_i\in \bm{V}$, we can show that the conditional probability distributions are positive as well i.e. $\mathbb{P}(v_i|\bm{pa}(v_i))>0$ for all $v_i\in C_{V_i}$ and setting $\bm{pa}(v_i)$ of the parents.
This positivity assumption is commonly assumed in many works related to causal graphs. For example, \cite{Hauser2012} assume positivity throughout their discussion when characterizing the Interventional Markov Equivalence class.
\end{remark}

Having stated these assumptions, we are now ready to state the main theorem of this paper. A detailed proof is provided in Section \ref{section:proof-main-theorem}.

\begin{theorem} 
\label{theorem:general-identify}
Let $\mathcal{G}=\{\bm{V},\mathcal{E}\}$ be a causal Bayesian Network and $\mathbb{P}(\bm{V})$ be the associated joint probability distribution satisfying Assumption \ref{assm:positivity}. Let $\mathbb{P}_{mix}(\bm{V})$ (Definition \ref{def:general-mix}) be any mixture of interventions that satisfies Assumption \ref{assm:intervention-exclusion}. The following are true.
\begin{enumerate}
    \item There exists a unique set of intervention tuples $\mathcal{T} = \{(\bm{t_1}, \pi_1), \ldots, (\bm{t_m}, \pi_m)\}$ satisfying Assumption \ref{assm:intervention-exclusion}, such that
    \[
    \mathbb{P}_{mix}(\bm{V}) = \sum\limits_{i=1}^m \pi_i \mathbb{P}_{\bm{t_i}}(V).
    \]
    \item Given access to $\mathcal{G}$,  $\mathbb{P}(\bm{V})$ and $\mathbb{P}_{mix}(\bm{V})$, there exists an algorithm, that runs in time $n*(m*k_{max})^{O(1)}$, and, outputs the set of intervention tuples $\mathcal{T}$ (satisfying Assumption \ref{assm:intervention-exclusion}) generating it. Here $n$ is the number of nodes in $\mathcal{G}$, $m$ is the size of set $\mathcal{T}$ and $k_{max}$ is the maximum number of distinct values that any node can take.
\end{enumerate}
\end{theorem}

\begin{remark}
Though Assumption \ref{assm:positivity} is a sufficient conditions for Theorem \ref{theorem:general-identify}, it is not necessary. In Example \ref{eg:positivity-counterexample}, we give an example that does not satisfy this assumption but is uniquely generated by a set of intervention tuples satisfying Assumption \ref{assm:intervention-exclusion}.
\end{remark}

\section{Proof of Main Theorem}
\label{section:proof-main-theorem}

In this section, we provide rigorous proof to both parts of Theorem \ref{theorem:general-identify} together. Our uniqueness proof (for Part $1$) is constructive and gives an algorithm as described in Part $2$. Our proof goes via an induction argument on the number of nodes $n$ present in the given Bayesian Network. There are many lemmas stated throughout the proof. For a cleaner exposition, all of their proofs are provided in Appendix \ref{appendix:lemma-sys-eqn}.

\subsection{Base Case ($n=1$)}
\label{subsection:base-case}
Consider a causal Bayesian Network $\mathcal{G} = (V, \mathcal{E})$ with only one vertex $V$ and no edges (i.e. $\mathcal{E}=\varnothing$), such that $\mathbb{P}(V)$ satisfies Assumption \ref{assm:positivity}. Let $C_V = \{v^1, \ldots,v^k\}$ be the set of values that $V$ can take. Therefore, by Assumption \ref{assm:positivity}, $\mathbb{P}(v^i)>0$ for all $i\in [k]$. Next, consider any mixture of interventions $\mathbb{P}_{mix}(V)$ that satisfies Assumption \ref{assm:intervention-exclusion}. Writing the most general form of $\mathbb{P}_{mix}$, i.e. allowing for scalar weights to be $\geq 0$, we can write,
\begin{equation*}
   \mathbb{P}_{mix}(V) = \pi_{0}\mathbb{P}_{t_0}(V) + \pi_{1}\mathbb{P}_{t_1}(V) + \ldots + \pi_{k}\mathbb{P}_{t_k}(V), 
\end{equation*}

where $t_0 = \varnothing, t_1 = v^1, \ldots ,t_k = v^k$. By the notation in Definition \ref{def:general-mix},  $\mathbb{P}_{\varnothing}(V) = \mathbb{P}(V)$. Subtracting $\mathbb{P}(V)$ from both sides and setting $\pi_0 = 1-\sum_{i=1}^k \pi_i$, we get,
\begin{equation*}
    \mathbb{P}_{mix}(V) - \mathbb{P}(V) = \sum\limits_{i=1}^k \pi_{i}(\mathbb{P}_{v^i}(V) - \mathbb{P}(V)).
\end{equation*}

Recall, from the definition of interventions in Section \ref{sec:prelim}, for any $v^j \in C_V$, $\mathbb{P}_{v^i}(v^j) = \delta_{v^i, v^j}$. Substituting $V=v^1, \ldots, v^k$ and using $\mathbb{P}_{v^i}(v^j) = \delta_{v^i, v^j}$, gives us $k$ linear equations which can be written in the following matrix form:
\begin{equation}
\label{eq:t-basecase-sys}
\left.
\left.
    \begin{bmatrix}
        1 - a_{1} & - a_{1} & . &. & - a_{1}\\
        - a_{2} & 1 - a_{2} & . & .& - a_{2}\\
        . & . & . &. &. \\
        - a_{k} & - a_{k} & . & . & 1 - a_{k}\\
    \end{bmatrix}
\right.
    \begin{bmatrix}
        \pi_1\\
        \pi_2\\
        .\\
        \pi_k\\
    \end{bmatrix}
\right.=
    \begin{bmatrix}
    b_1\\
    b_2\\
    .\\
    b_k\\
    \end{bmatrix}
\end{equation}
where $b_i=\mathbb{P}_{mix}(v^i)-\mathbb{P}(v^i)$ and $a_i=\mathbb{P}(v^i)>0$ (Assumption \ref{assm:positivity}).
Any set of intervention tuples $\mathcal{T}$ generating $\mathbb{P}_{mix}(V)$ can be obtained as a solution to the above system. Since, in Part $1$ of the theorem we restrict our focus to $\mathcal{T}$ that satisfy Assumption \ref{assm:intervention-exclusion}, we know there exists some $i\in [k]$, such that $\pi_i=0$. In the following lemma, we show that such a system under these assumptions has a unique solution when $\pi_1, \ldots,\pi_k \in \mathbb{R}_{\geq 0}$. Proof of this lemma is presented in Appendix \ref{proof:lemma-sys-eqn}.

\begin{lemma}
\label{lemma:sys-eqn}
Consider the following linear system.
\begin{equation*}
\label{eq:l-base-sys}
\left.
\left.
    \begin{bmatrix}
        c - a_{1} & - a_{1} & . &. & - a_{1}\\
        - a_{2} & c - a_{2} & . & .& - a_{2}\\
        . & . & . &. &. \\
        - a_{k} & - a_{k} & . & . & c - a_{k}\\
    \end{bmatrix}
\right.
    \begin{bmatrix}
        x_1\\
        x_2\\
        .\\
        x_k\\
    \end{bmatrix}
\right.=
    \begin{bmatrix}
    b_1\\
    b_2\\
    .\\
    b_k\\
    \end{bmatrix}
\end{equation*}
Assume that $a_1,\ldots,a_k>0$, $\sum_{j=1}^{k}a_j=c$ and it has at least one solution. Then, rank of the above matrix is $k-1$ and there are infinitely many solutions. Under the assumption that $\bm{x}\in \mathbb{R}_{\geq 0}$ and $x_i=0$ for some $i\in[k]$, the solution becomes unique. Given access to $a_j$s, $b_j$s and $c$, there exists an algorithm that computes this solution in $k^{O(1)}$ time.
\end{lemma}
 It's easy to see that Equation \ref{eq:t-basecase-sys} satisfies all requirements of Lemma \ref{lemma:sys-eqn}, implying the base case of our induction proof. 

\textbf{Inductive hypothesis $(n=N)$:} Assume, Theorem \ref{theorem:general-identify} is true for all causal Bayesian Networks on $N$ nodes, that satisfy Assumption \ref{assm:positivity} and input mixtures that satisfy Assumption \ref{assm:intervention-exclusion}.

\subsection{Induction step $(n=N+1)$:}
\label{subsection:induction-step}
Assuming the above inductive hypothesis, we show that Theorem \ref{theorem:general-identify} is true for all causal Bayesian Networks on $N+1$ nodes, and mixture of interventions on it, satisfying Assumptions \ref{assm:positivity} and  \ref{assm:intervention-exclusion} respectively. Let $\bm{V} = \{V_1,\ldots, V_{N+1}\}$, $\mathbb{P}(\bm{V})$ be the distribution of $\bm{V}$ and $\mathbb{P}_{mix}(\bm{V})$ be any mixture of interventions that satisfies Assumption \ref{assm:intervention-exclusion}.
We wish to show that there is a unique set of intervention tuples satisfying Assumption \ref{assm:intervention-exclusion} that generates $\mathbb{P}_{mix}(\bm{V})$. Without loss of generality let $V_1\prec \ldots \prec V_{N+1}$ be a topological order for $\mathcal{G}$. We will now marginalize on $V_{N+1}$ to reduce our problem to the $n=N$ case, so that we can use the inductive hypothesis. The following lemma is required to make this argument. We present it's proof in Appendix \ref{proof:lemma-marginalization}.
\begin{lemma}
\label{lemma:marginalization}
Let $\bm{V}_N = \{V_1,\ldots,V_N\}$,
\begin{enumerate}
    \item $\mathbb{P}(\bm{V}_N)$ is generated by the CBN $\mathcal{G}_N$ obtained by  deleting vertex $V_{N+1}$ (and all its incoming edges) from $\mathcal{G}$,
    and satisfies Assumption \ref{assm:positivity}.
    \item $\mathbb{P}_{mix}(\bm{V}_N)$ can be written as a mixture of interventions on $\mathcal{G}_N$ that satisfies Assumption \ref{assm:intervention-exclusion}.
    \item Given access to $\mathbb{P}(\bm{V})$ and  $\mathbb{P}_{mix}(\bm{V})$, in $O(k_{max})$ time we can create access to $\mathbb{P}(\bm{V}_N)$ and $\mathbb{P}_{mix}(\bm{V}_N)$, by marginalizing on $V_{N+1}$.
\end{enumerate}
\end{lemma}

Using the inductive hypothesis with this claim, we get that there exists a unique set of intervention tuples $\mathcal{S} = \{(\bm{s}_1, \mu_1), \ldots, (\bm{s}_q, \mu_q)\}$\footnote{For $i\in [q]$, $\bm{s}_i$ are values taken by variables $\bm{S}_i\subset \bm{V}$} satisfying Assumption \ref{assm:intervention-exclusion} that generates $\mathbb{P}_{mix}(\bm{V}_N)$, i.e.,
\begin{equation*}
\label{eq:t-marg-raw}
\mathbb{P}_{mix}(\bm{V}_N) = \sum\limits_{j=1}^q\mu_j \mathbb{P}_{\bm{s}_j}(\bm{V}_N),
\end{equation*}
The induction hypothesis also implies that $\mathcal{S}$ can be computed in $N*(q*k_{max})^{O(1)}$ time using access to $\mathbb{P}(\bm{V}_N)$ and $\mathbb{P}_{mix}(\bm{V}_N)$.
The next step in our proof then is to show that, for a given $\mathcal{G}$, $\mathbb{P}(\bm{V})$ and $\mathbb{P}_{mix}(\bm{V})$, the set of intervention tuples $\mathcal{S}$ can be uniquely \emph{lifted} to a set $\mathcal{T}$ of intervention tuples that satisfies Assumption \ref{assm:intervention-exclusion} and generates $\mathbb{P}_{mix}(\bm{V})$. We also show that using access to $\mathcal{G}$, $\mathbb{P}(\bm{V})$ and $\mathbb{P}_{mix}(\bm{V})$, the lifting process runs in $(m*k_{max})^{O(1)}$ time implying that $\mathcal{T}$ can be computed in $(N+1)*(m*k_{max})^{O(1)}$ time.

\subsubsection{Lifting $\mathcal{S}$}
\label{subsubsection:lifting}
In this section we lift the set of intervention tuples  $\mathcal{S}$ generating $\mathbb{P}_{mix}(\bm{V}_N)$ uniquely to a set of intervention tuples satisfying Assumption \ref{assm:intervention-exclusion} generating $\mathbb{P}_{mix}(\bm{V})$. Let $\mathcal{T} = \{(\bm{t}_1, \pi_1), \ldots, (\bm{t}_m, \pi_m)\}$ be any arbitrary set of intervention tuples satisfying Assumption \ref{assm:intervention-exclusion} that generates $\mathbb{P}_{mix}(\bm{V})$, i.e.

\begin{equation}
\label{equation:input-mixture}
\mathbb{P}_{mix}(\bm{V}) = \sum\limits_{i=1}^m\pi_i \mathbb{P}_{\bm{t}_i}(\bm{V}).
\end{equation} 

First, we give a lemma that connects targets $\bm{t}_1, \ldots \bm{t}_m$ inside $\mathcal{T}$ with targets $\bm{s}_1,\ldots \bm{s}_q$ inside $\mathcal{S}$. We present it's proof in Appendix \ref{proof:lemma-lifting}.

\begin{lemma}
\label{lemma:lifting}
For every $\bm{t}_i, i\in [m]$, there is some $\bm{s}_j, j\in [q]$ such that, either $\bm{t}_i = \bm{s}_j$ or $\bm{t}_i = \bm{s}_j\cup \{v\}$ for some $v$ in $C_{V_{N+1}}$.
\end{lemma}
For $j\in [q]$, we define sets $\mathcal{S}_j = \{\bm{s}_j, \bm{s}_j\cup\{v^1\}, \ldots \bm{s}_j\cup\{v^k\}\}$ where $C_{V_{N+1}} = \{v^1, \ldots, v^k\}$. Since the targets $\bm{s}_j, j\in [q]$ are distinct, the sets $\mathcal{S}_j, j\in [q]$ are disjoint. Lemma \ref{lemma:lifting} implies that 
\[
\{\bm{t}_1, \ldots,\bm{t}_m\} \subset \mathcal{S}_1\cup\ldots\mathcal{S}_q
\]
Since $\mathcal{T}$ was arbitrary, for every such $\mathcal{T}$, there exist non-negative scalars $\pi_{\bm{s}}$, $\bm{s}\in \mathcal{S}_1\cup\ldots\cup\mathcal{S}_q$ such that Equation \ref{equation:input-mixture} can be written as,
\begin{equation}
    \label{equation:lifted-mixture}
    \mathbb{P}_{mix}(\bm{V}) = \sum\limits_{j=1}^q\sum\limits_{\bm{s}\in \mathcal{S}_j} \pi_{\bm{s}}\mathbb{P}_{\bm{s}}(\bm{V})
\end{equation}
Any solution of Equation \ref{equation:lifted-mixture} with $\pi_{\bm{s}}\geq 0$ gives a set of intervention tuples for $\mathbb{P}_{mix}$. We show that there is a unique such set which satisfies Assumption \ref{assm:intervention-exclusion}.

\begin{lemma}
Let $\pi_{\bm{s}}$, $\bm{s}\in \mathcal{S}_1\cup\ldots\cup\mathcal{S}_q$ be some non-negative solution to Equation \ref{equation:lifted-mixture} and the set $\mathcal{T} = \{(\bm{s}, \pi_{\bm{s}}) : \pi_{\bm{s}}>0\}$ be the corresponding set of intervention tuples. There exists a unique $\mathcal{T}$ that satisfies Assumption \ref{assm:intervention-exclusion}.
\end{lemma}

\begin{proof}
 We show that enforcing Assumption \ref{assm:intervention-exclusion} uniquely determines all $\pi_{\bm{s}}$ as solutions to a sequence of system of linear equations, implying that there is a unique $\mathcal{T}$  that satisfies Assumption \ref{assm:intervention-exclusion}. To construct this sequence, we need an ordering on $\bm{s}_1, \ldots, \bm{s}_q$. So, without loss of generality, we assume that for $j_1\leq j_2$, $\bm{s}_{j_2} \not\subseteq \bm{s}_{j_1}$. The linear equations are created by using specific settings for $\bm{V}$ in Equation \ref{equation:lifted-mixture} which enable us to decompose the linear system into a sequence of simpler systems i.e. one for each $\mathcal{S}_i$. We propose these settings next and explain why and how they work. Since $\mathcal{S}$ satisfies Assumption \ref{assm:intervention-exclusion}, there exists $\Bar{v}_i\in V_i, i\in [N]$ such that for all $j\in [q]$, $\Bar{v}_i\notin \bm{s}_j$. For $j\in [q]$, we define $\bm{s}_{-j} = \{\Bar{v}_i: V_i\notin \bm{S}_j\}$\footnote{$\bm{s}_j$ corresponds to set of variables $\bm{S}_j\subset \bm{V}$.} and for every $l\in [k]$, create settings
\[
\bm{v}_{j,l} = \bm{s}_j\cup \bm{s}_{-j}\cup\{v^l\}
\]
 where $C_{V_{N+1}} = \{v^1, \ldots, v^k\}$. The following lemma is used to decompose the system of equations into simpler systems. Proof is presented in Appendix \ref{proof:lemma-setting}.
\begin{lemma}
\label{lemma:setting}
For $i\in [q]$, $l\in [k]$ and $\bm{s} \in \mathcal{S}_{i+1}\cup\ldots\cup\mathcal{S}_q$,
\[
\mathbb{P}_{\bm{s}}(\bm{v}_{i,l}) = 0
\]
\end{lemma}
Using this in Equation \ref{equation:lifted-mixture}, leaves us with the following simpler system for every $i\in [q]$,
\begin{equation}
\label{equation:reduced-si}
   \mathbb{P}_{mix}(\bm{v}_{i,l}) - \sum\limits_{j=1}^{i-1}\sum\limits_{\bm{s}\in \mathcal{S}_j} \pi_{\bm{s}}\mathbb{P}_{\bm{s}}(\bm{v}_{i,l})  = \sum\limits_{\bm{s}\in \mathcal{S}_i} \pi_{\bm{s}}\mathbb{P}_{\bm{s}}(\bm{v}_{i,l}) 
\end{equation}

Suppose all $\pi_{\bm{s}}$, $\bm{s}\in \mathcal{S}_1\cup\ldots\cup\mathcal{S}_{i-1}$, have been determined. Then the left had side of this equation is completely known and has no unknown variables. We denote it by $\Delta$ going forward. Therefore, by varying $l\in [k]$, we have $k$ equations in $k+1$ variables $\pi_{\bm{s}}, \bm{s}\in \mathcal{S}_i$. In the next lemma, we will obtain a linear equation satisfied by these $k+1$ variables and reduce the system to $k$ equations in $k$ variables. On marginalizing over $V_{N+1}$ in Equation \ref{equation:lifted-mixture}, we get 
\begin{lemma}
\label{lemma:scalar-relationship}
For all $i\in [q]$ and $(\bm{s}_i,\mu_i)\in \mathcal{S}$, the following holds.
\[
    \mu_i = \sum\limits_{\bm{s}\in \mathcal{S}_i}\pi_{\bm{s}}
\]
\end{lemma}
Proof of Lemma \ref{lemma:scalar-relationship} is presented in Appendix \ref{proof:lemma-scalar-relationship}.
By making the substitution from this lemma above into Equation \ref{equation:reduced-si}, we get the equation
\begin{equation}
\label{equation:further-reduced-si}
   \Delta - \mu_i \mathbb{P}_{\bm{s}_i}(\bm{v}_{i,l})  = \sum\limits_{l^\prime\in [k]} \pi_{\bm{s}_i\cup\{v^{l^\prime}\}}(\mathbb{P}_{\bm{s}_i\cup\{v^{l^\prime}\}}(\bm{v}_{i,l}) - \mathbb{P}_{\bm{s}_i}(\bm{v}_{i,l})) 
\end{equation}
that gives a system of $k$ equations in $k$ variables when we vary $l^\prime\in [k]$. Clearly we are looking for non-negative solutions for $\pi_{\bm{s}_i\cup\{v^{l^\prime}\}}$, $l^\prime\in [k]$. When we enforce Assumption \ref{assm:intervention-exclusion}, there is some $l^\prime\in [k]$ such that $\pi_{\bm{s}_i\cup\{v^{l^\prime}\}}=0$. In Lemma \ref{lemma:si-choice}, we show that we can uniquely solve Equation \ref{equation:further-reduced-si} for such $\pi_{\bm{s}_i\cup\{v^{l^\prime}\}}$, $l^\prime\in [k]$.
\begin{lemma}
\label{lemma:si-choice}
For every $i\in [q]$, Equation \ref{equation:further-reduced-si} has a unique solution 
when we enforce that $\pi_{\bm{s}_i\cup\{v^{l^\prime}\}}$, $l^\prime\in [k]$ are non-negative and at least one of them is $0$.
\end{lemma}
We present a proof of this Lemma in Appendix \ref{proof:lemma-si-choice}.
This lemma implies that under enforcement of Assumption \ref{assm:intervention-exclusion}, all targets $\in \mathcal{S}_i$ (and their respective mixing coefficients) that appear in $\mathbb{P}_{mix}(\bm{V})$  get uniquely identified. Using this technique from $i=1$ to $q$, any set of intervention tuples satisfying Assumption \ref{assm:intervention-exclusion} that generates $\mathbb{P}_{mix}(\bm{V})$ gets uniquely identified. Therefore, there is a unique set of intervention tuples $\mathcal{T}$ that generates $\mathbb{P}_{mix}$ and satisfies Assumption \ref{assm:intervention-exclusion}.
\end{proof}
The lifting of targets in $\mathcal{S}_i$ is done in Lemma \ref{lemma:si-choice} using technique from  Lemma \ref{lemma:sys-eqn} which takes $(k_{max})^{O(1)}$ time. This is repeated for all $i\in [q]$, therefore, we spend $(q*k_{max})^{O(1)}$ time. It's easy to see that $q \leq m$ and so using the induction hypothesis the set of intervention tuples is computed in  $(N+1)*(m*k_{max})^{O(1)}$ time, completing the induction step. We describe our complete algorithm in Algorithm \ref{algo:general-algo-inf}. It's correctness and time complexity follows from the discussion in this section. For better understanding, in Examples \ref{eg:workout-example} and \ref{eg:workout-example-2}, we provide two worked out examples on small problem instances, that illustrate important aspects of our algorithm.

\begin{algorithm}[t]
\caption{DISENTANGLE-INFINITE}
\label{algo:general-algo-inf}
\SetAlgoLined
\DontPrintSemicolon
\SetKwInOut{Input}{input}
\SetKwInOut{Output}{output}
\SetKw{Return}{return}

\Input{Variables $\bm{V} = (V_1, \ldots, V_{N+1})$, CBN $\mathcal{G}$, Distributions $\mathbb{P}(\bm{V}), \mathbb{P}_{mix}(\bm{V})$}
\Output{Set of intervention tuples $\mathcal{T}$}
\BlankLine

\begin{enumerate}
    \item When $|\bm{V}|=1$, setup the linear system in Equation \ref{eq:t-basecase-sys} and solve it using technique described in Lemma \ref{lemma:sys-eqn} to obtain a set $\mathcal{T}$ of intervention tuples. \Return $\mathcal{T}$. 
    
    \item Let $V_1\prec\ldots\prec V_{N+1}$ denote a topological order in $\mathcal{G}$. Marginalize on $V_{N+1}$ to create access to $\mathbb{P}_{mix}(\bm{V}_N)$ and $\mathbb{P}(\bm{V}_N)$ where $\bm{V}_N = (V_1, \ldots,V_N)$. Construct $\mathcal{G}_N = \mathcal{G}\setminus \{V_{N+1}\}$. Recursively call this algorithm with inputs $\mathcal{G}_N$, $\mathbb{P}(\bm{V}_N)$, $\mathbb{P}_{mix}(\bm{V}_N)$, to compute the unique set of intervention tuples $\mathcal{S} = \{(\bm{s}_1, \mu_1), \ldots, (\bm{s}_q, \mu_q)\}$ that satisfies Assumption \ref{assm:intervention-exclusion} and generates $\mathbb{P}_{mix}(\bm{V}_N)$. Let $\bm{s}_1, \ldots, \bm{s}_q$ be ordered such that $i\leq j$ implies that $\bm{s}_j \not\subseteq \bm{s}_i$.
    For each $i\in [N]$, by inspecting $\bm{s}_j, j\in [q]$, identify $\Bar{v}_i\in C_{V_i}$ such that $\Bar{v}_i\notin \bm{s}_j$ for any $j\in [q]$.
    Define $\bm{s}_{-j} = \{\Bar{v}_i : V_i\notin S_j\}$. Let $C_{V_{N+1}} = \{v^1, \ldots, v^k\}$. For each $i\in[q]$ and $l\in [k]$, create setting $\bm{v}_{i,l} = \bm{s}_i\cup \bm{s}_{-i}\cup\{v^l\}$. 
    
    \item For each fixed $i\in [q]$, evaluate distributions for different $\bm{v}_{i,l}$, $l\in [k]$, to setup the system of equations described in Equation \ref{equation:further-reduced-si}. Solve the system using the technique outlined in proof of Lemma \ref{lemma:si-choice} (which in turn uses Lemma \ref{lemma:sys-eqn}). At the end of this process collect all the intervention tuples thus obtained (for all $i\in [q]$), in the set $\mathcal{T}$. \Return $\mathcal{T}$.
\end{enumerate}
\end{algorithm}

\section{Simulation Study}
\label{sec:exp}

The main purpose of this simulation study is to experimentally analyze the performance of Algorithm \ref{algo:general-algo} which modifies Algorithm \ref{algo:general-algo-inf} to make it work with finite number of samples  from distributions $\mathbb{P}_{mix}(\bm{V})$ and $\mathbb{P}(\bm{V})$.

\textbf{Simulation Setup:} For each simulation setting $(N,M)$\footnote{$N$ is number of nodes, $M$ is number of samples} we randomly sample a directed acyclic graph on $N$ nodes (each having $3$ categorical values), from the Scale-Free (SF) model (\cite{ScaleFree}), with number of edges chosen uniformly randomly from the set of integers $[N,5N]$. Given this graph, we model the conditional probability distribution of each node as a multinoulli distribution with Dirichlet priors having fixed parameter $\alpha=2$ for all categories. This is done to conform with Assumption \ref{assm:positivity}. This generates our causal Bayesian Network $\mathcal{G}$. We estimate marginal probabilities of the joint distribution defined by $\mathcal{G}$ by generating $M$ samples using ancestral sampling on the network. Now, to create input instances, we first choose an integer $m$ uniformly randomly from the set $[4, 16]$ and use this as the number of intervention tuples in the mixture. Then we iterate from $1$ to $m$ to build each intervention target of the mixture. First, we choose the size of the target by picking an integer $r$ uniformly randomly from the set $\{0,\ldots,N\}$. Then we uniformly randomly choose an $r-$sized subset of $[N]$, defining the variables in the current target. For each of these variables, we first choose a category uniformly randomly and remove it from consideration (in order to satisfy Assumption \ref{assm:intervention-exclusion}). From the remaining categories, we uniformly randomly select one for each of the variables in the target. Finally, we generate $m$ scalar weights for the mixing coefficients such that they sum to $1$. In order to make sure that these coefficients are not too small, we generate them with Dirichlet priors with all parameter values fixed to $2$. The settings for number of nodes $N$ and  sample size $M$ used in the experiment  are $(N,M) \in \{4,8,12\} \times \{2^4,2^5,\ldots,2^{20}\}$ where $\times$ is the direct product of sets.


\begin{figure}
  \begin{subfigure}[t]{.45\textwidth}
    \centering
    \includegraphics[width=\linewidth]{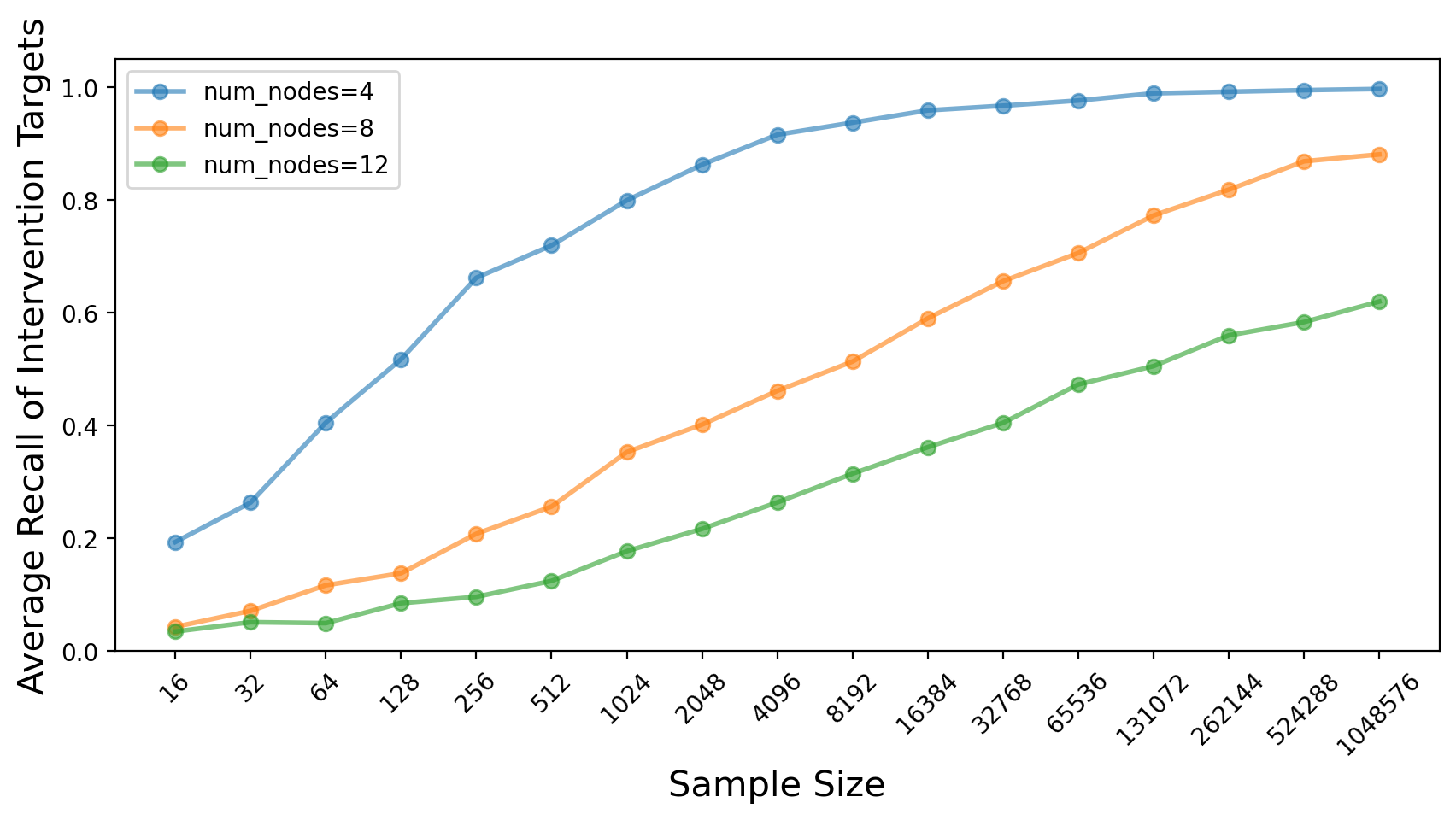}
		\caption[]%
        {{\small Recall}}    
		\label{fig:recall}
  \end{subfigure}
  \hfill
  \begin{subfigure}[t]{.45\textwidth}
    \centering
    \includegraphics[width=\linewidth]{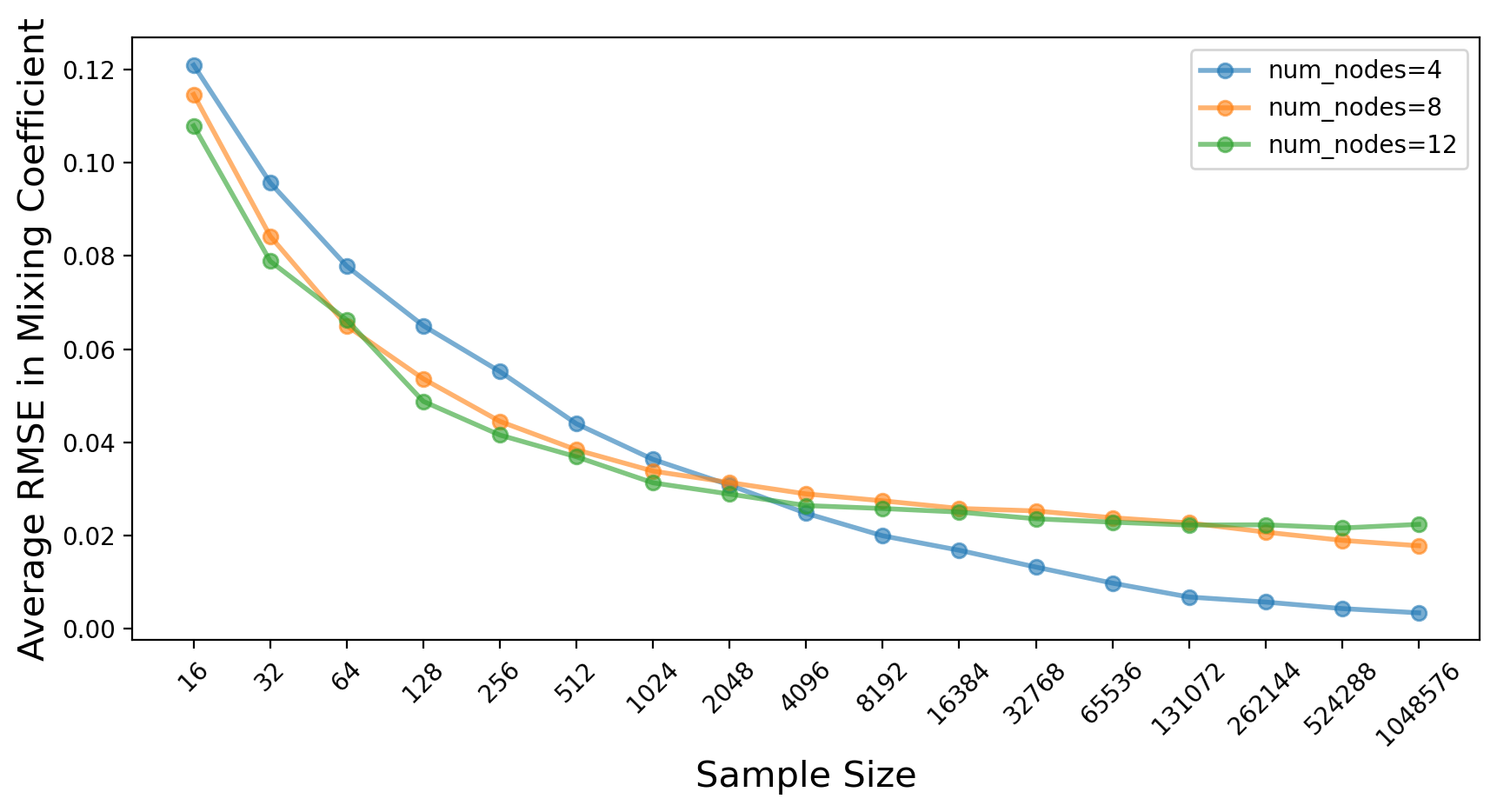}
		\caption[]%
        {{\small RMSE}}    
		\label{fig:rmse} 
  \end{subfigure}

  \medskip

  \begin{subfigure}[t]{.45\textwidth}
    \centering
    \includegraphics[width=\linewidth]{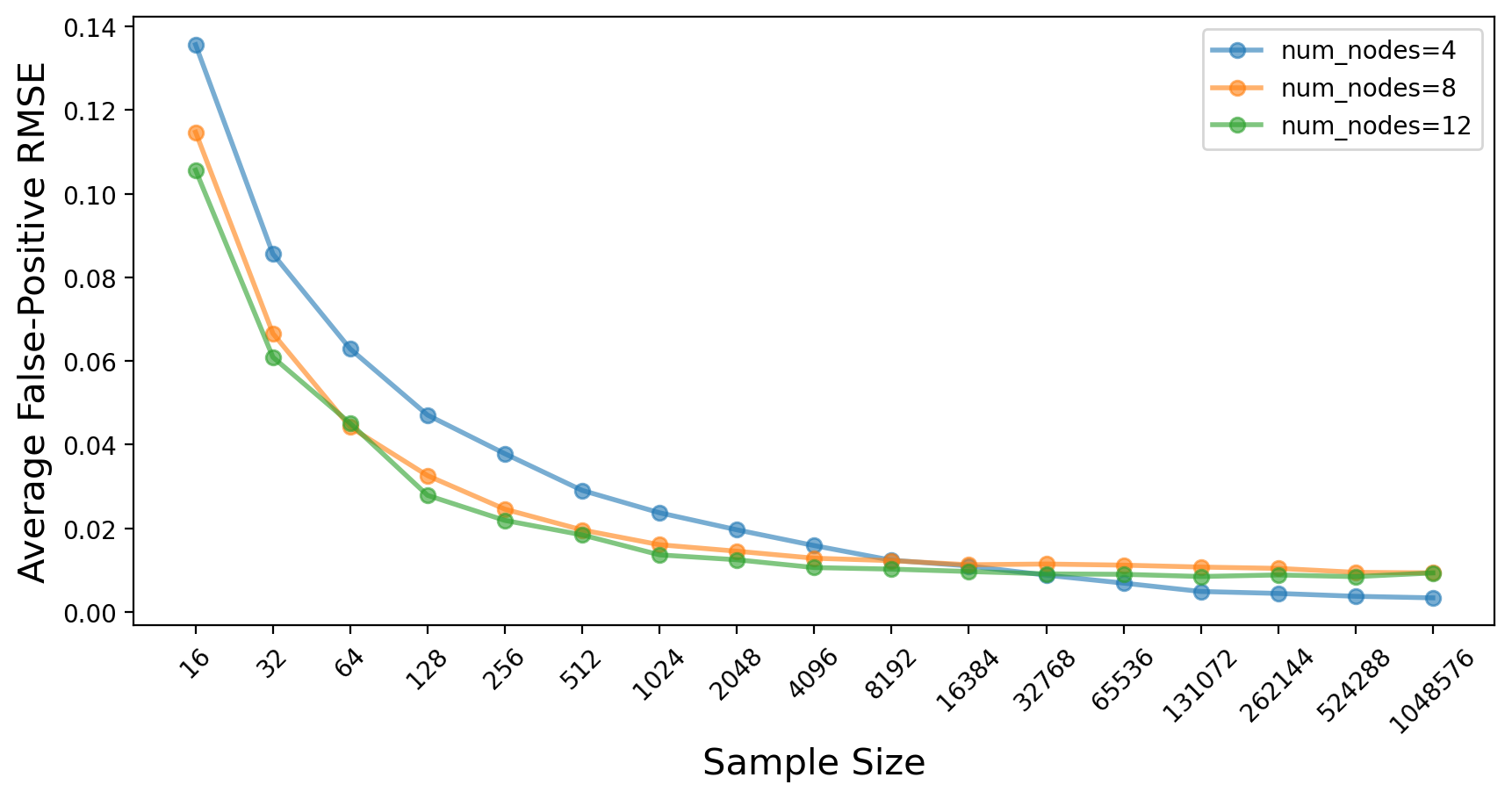}
		\caption[]%
        {{\small False-Positive RMSE}}   
		\label{fig:fp-rmse}
  \end{subfigure}
  \hfill
  \begin{subfigure}[t]{.45\textwidth}
    \centering
    \includegraphics[width=\linewidth]{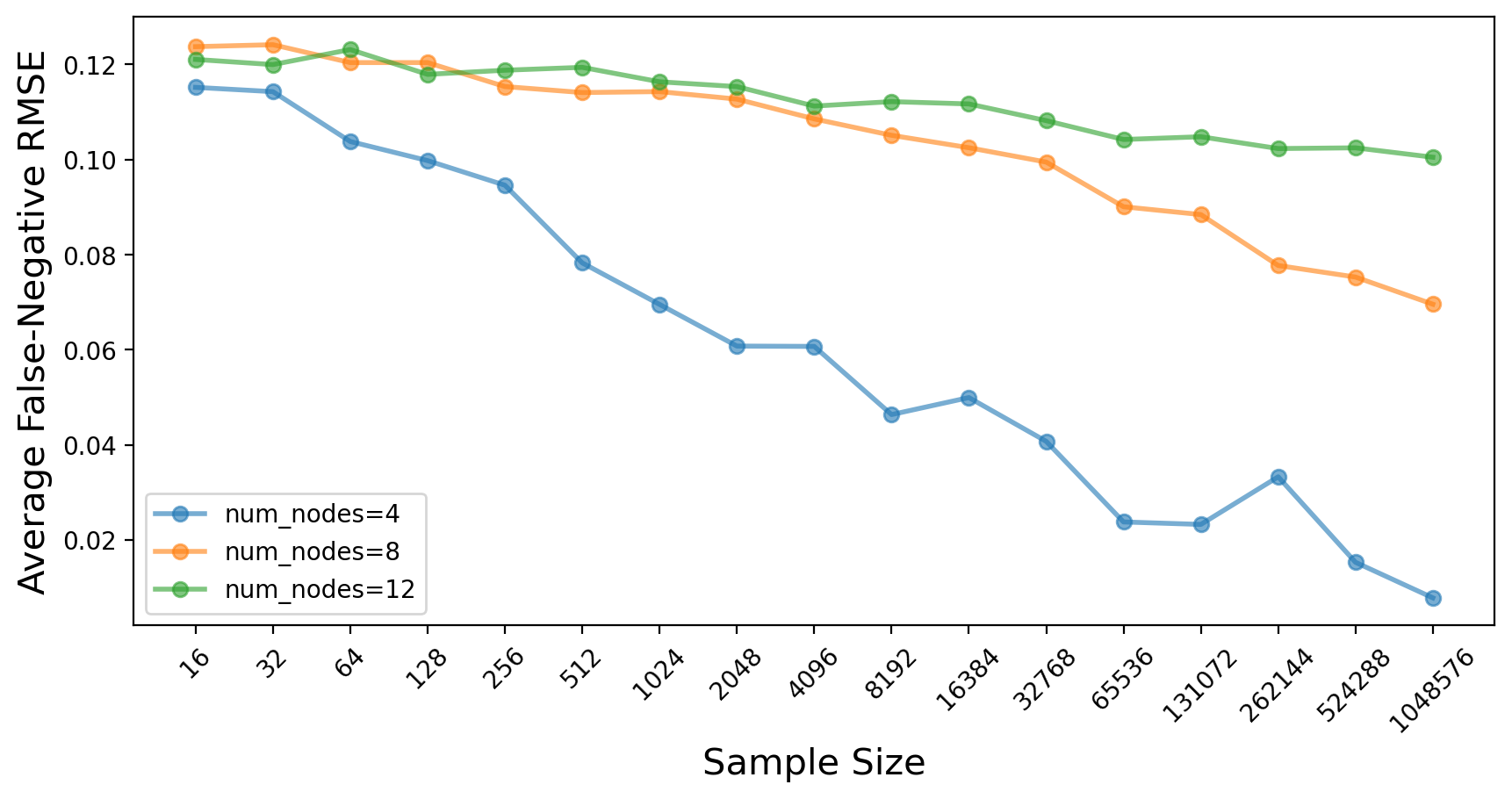}
		\caption[]%
        {{\small False-Negative RMSE}}    
		\label{fig:fn-rmse}
  \end{subfigure}
  \caption{Performance of Algorithm \ref{algo:general-algo} as sample size and number of nodes increase} 
\label{fig:sample-variation}
\end{figure}


\textbf{Results Discussion:} Figure \ref{fig:sample-variation} presents four plots that demonstrate
performance of our algorithm as sample size $M$ varies in $\{2^4,2^5,\ldots,2^{20}\}$. We also vary the number of nodes $N$ in $\{4,8,12\}$ and show separate plots for each $N$ in each of the figures. The four plots in Figure \ref{fig:sample-variation}, demonstrate four different accuracy metrics we describe in Appendix \ref{appendix:exp-metrics}. In Figure \ref{fig:recall}, we plot the average recall of intervention targets as $M$ increases. Recall for a single input instance is the number of intervention targets in the input that are identified in the output, as defined in Appendix \ref{appendix:exp-metrics}. Average recall is the average of this over all random instances generated in the simulation. We observe a general trend of increase in the recall as we increase the number of samples. Also, a relatively larger number of samples are required to achieve the same level of recall for mixtures generated from CBN with a large number of nodes as compared to smaller ones. This trend is expected as Algorithm \ref{algo:general-algo} estimates the intervention targets by sequentially adding nodes to them. Hence for larger-sized CBNs, the error accumulated is larger as compared to smaller ones.

In Figure \ref{fig:rmse}, we plot the average root-mean-squared error (RMSE) between the estimated and actual mixing coefficients. For each input, RMSE is calculated using the definition supplied in Appendix \ref{appendix:exp-metrics}. Then it is averaged over all the random input instances. We 
observe a fast decrease in the average RMSE as $M$ increases. We also observe that the average RMSE is higher for higher $N$. This is also expected since for distributions on larger number of variables, more samples will be needed to estimate marginal probabilities accurately.

In Figure \ref{fig:fp-rmse}, we plot average False-Positive RMSE (Section Appendix \ref{appendix:exp-metrics}) or FP-RMSE as $M$ increases. For each input instance, FP-RMSE computes the RMSE in mixing proportions for components which are not present in actual target set but predicted by our algorithm. This is then averaged over all the random input instances. For each value of $N$, we observe a similar decreasing trend in this plot showing that incorrect targets in our output have very small mixing proportions (as sample size increases) and therefore even if they are present in the output their contribution is insignificant.

In Figure \ref{fig:fn-rmse}, we plot average False-Negative RMSE (Section Appendix \ref{appendix:exp-metrics}) or FN-RMSE as $M$ increases. For each input instance, FN-RMSE computes the RMSE in mixing proportions for components present in the actual target but not present in the output targets. This is then averaged over all the random input instances. Even though we observe a clear decreasing trend in this situation as well, the rate is much slower as $N$ increases. This implies that the sample complexity of our algorithm is high and it might need too many samples to correctly identify the coefficients of targets present in the input. Reducing the sample complexity is an interesting research direction which we plan to pursue in a future work.

In Figure \ref{fig:graph-comparison}, we demonstrate and compare performance of our Algorithm (as $M,N$ increase), for CBNs generated using different random graph models (Scale-Free and Erd\"os-R\'enyi). We observe no significant difference in performance and make a conjecture that only high level graph parameters (such as number of nodes, edges, in-degree etc.) might be having an impact on performance and the topology (given these parameters) might not be that crucial.

To further understand the performance  of our algorithm with respect to the number of nodes, in Figure \ref{fig:node-variation}, we plot the Average Recall and Average RMSE as number of nodes varies from $4$ to $32$, for a fixed sample size of $\sim10^6$. We observe that recall decreases and RMSE increases very quickly as number of nodes increase. Even though this is expected since error is accumulated as we successively add nodes and find new intervention targets, such performance for a very large sample size indicates bad dependence of sample complexity on the number of nodes. Improving this needs more exploration and is left for future work.

\textbf{Limitations and Future Directions:} The increasing trend in recall and decreasing trend in RMSE of mixing coefficients shows promise. But the current algorithm appears to be expensive in terms of sample complexity, especially for mixture generated from larger graphs as seen in Figures \ref{fig:recall}, \ref{fig:fn-rmse} and Figure \ref{fig:node-variation}. Hence, it will be interesting to explore directions which could reduce sample complexity. We leave this for future work. Another limitation is the absence of baseline works to compare to. Since, ours is the first paper that proves identifiability of such mixtures and gives the first such algorithm, there are no prior works to compare against. In future, we plan to compare our algorithm on a related or downstream task that might have been explored in other works such as \cite{Thiesson1998, Squires2020, Jaber2020}.

\section{Conclusion}
\label{sec:conclusion}
In this paper, we investigated the problem of identifying individual intervention targets from a mixture of interventions on a causal Bayesian Network. This problem is well motivated from the real-world scenario wherein experiments/interventions are accompanied by stochastic hidden off-target effects. We modeled this problem as a mixture of intervention distributions and constructed examples to show that, in general, it is impossible to identify all targets in it. Then, we proposed a mild \emph{positivity} assumption on the underlying network and a very reasonable \emph{exclusion} assumption on the intervention targets that can appear in the mixture distribution. Using these assumptions we proved that given access to the underlying CBN and the mixture distribution, there is a unique set of intervention targets that satisfies our \emph{exclusion} assumption and also generates the mixture. Our uniqueness proof also provides an algorithm that uses access to the underlying distributions and efficiently identifies all the targets along with their coefficients in the mixture. In order to work with finitely many samples from the distributions, we created a small modification to our algorithm and validated it's performance using simulated experiments. We tested our algorithm and bench-marked its performance as the number of samples and nodes increased. As future work, we plan to investigate algorithms to recover targets in such mixtures using a smaller number of samples.
Another interesting direction is to use limited access to the underlying CBN while recovering the targets. This can be very useful in situations where sufficient data or prior knowledge might not be available to pin down the CBN. Solving the identifiability problem when the CBN has unobserved confounders might be a good first step in this direction.
\section{Acknowledgements}
This work was done during an internship of the first author (Abhinav Kumar) under the guidance and mentorship of the second author (Gaurav Sinha), during the period January-August 2020. Abhinav Kumar would like to thank Adobe Research India for hosting him during this period and providing facilities for a fruitful and enjoyable research experience. Abhinav Kumar did this internship as part of his undergraduate thesis offered by his undergraduate institution BITS Pilani, Hyderabad. He would like to thank the institution for this opportunity. Gaurav Sinha would like to thank his mentees Aurghya Maiti, Pulkit Goel, Naman Poddar and Ayush Chauhan who were part of an older internship where the seed of this work was planted. The authors would like to thank anonymous reviewers for their very helpful comments which helped in greatly improving the presentation of this paper.

\bibliographystyle{alpha}
\bibliography{main}

\appendix
\section{Proofs from Section \ref{section:proof-main-theorem}}
\label{appendix:lemma-sys-eqn}
In this section, we provide missing proofs of the lemmas stated in Section \ref{section:proof-main-theorem}. 

\subsection{Proof of Lemma \ref{lemma:sys-eqn}}
\label{proof:lemma-sys-eqn}
Note that we have assumed $a_i > 0$ for all $i\in [k]$. We iterate from $i=1$ to $k-1$ and apply the following row transformations on our matrix.
\begin{equation*}
   \begin{split}
    R_i&\mapsto R_i-(\frac{a_i}{a_k})R_k\\
\end{split} 
\end{equation*}

This results in the following linear system.
\begin{equation*}
\label{eq:sys-rref}
    \left.
        \left.
            \begin{bmatrix}
            c & 0 & . & -\frac{a_1}{a_k}c \\
            0 & c & . & -\frac{a_2}{a_k}c \\
            . & . &. & . \\
            . & . &. & . \\
            -a_k & -a_k & . & c-a_k \\
            \end{bmatrix}
        \right.
            \begin{bmatrix}
            x_1\\
            x_2\\
            .\\
            .\\
            x_k\\
            \end{bmatrix}
    \right.=
        \begin{bmatrix}
        \tilde{b}_1\\
        \tilde{b}_2\\
        .\\
        .\\
        \tilde{b}_k\\
        \end{bmatrix}
\end{equation*}
where $\tilde{b}_i=(b_i-\frac{a_i}{a_k}b_k)$ for all $i\in[k-1]$ and $\tilde{b}_k=b_k$. Since $c>0$, this matrix is easily seen to have rank $\geq k-1$. Using $c =\sum_{i=1}^k a_i$ we can easily check that the last row $R_k$ is $\frac{-a_k}{c}(R_1+\ldots + R_{k-1})$, implying that it has rank $k-1$.  Since the system is assumed to have at least one solution, it actually has infinitely many solutions. The null space of this matrix is the one dimensional space spanned by $\bm{w} = (\frac{a_1}{a_k}, \ldots, \frac{a_k}{a_k})^T$ which has all positive entries since $a_i > 0, i\in [k]$. Assume there are two distinct solutions $\bm{u}= (u_1,\ldots,u_k)^T$ and $\bm{v} = (v_1,\ldots,v_k)^T$ in $\mathbb{R}_{\geq 0}^k$ such that both have at least one of their co-ordinates $0$, then $\bm{u}-\bm{v}$ belongs to the null space i.e. $\bm{u} - \bm{v} = \lambda \bm{w}$ for some non-zero scalar $\lambda$. If the same co-ordinate of $\bm{u}, \bm{v}$ are $0$ i.e. for some $i\in [k]$, $u_i = v_i = 0$, then since $\frac{a_i}{a_k}$ is non-zero $\lambda=0 \Rightarrow \bm{u} = \bm{v}$, a contradiction. If different co-ordinates of $\bm{u}, \bm{v}$ are $0$, say $u_1=v_2=0$, then since $\bm{u}, \bm{v} \in \mathbb{R}_{\geq 0}^k$, $u_1-v_1$ is negative and $u_2-v_2$ is positive. This is not possible since both these quantities should have the same sign as $\lambda$, as all co-ordinates of $\bm{w}$ are strictly positive. Therefore we arrive at a contradiction and there is a unique solution.

Having proved this uniqueness, finding the solution is easy. We perform the above-mentioned row transformations and obtain the general solution. Then for each $i\in [k]$, we set $x_i=0$ and try to solve for the other variables $x_1,\ldots,x_{i-1}, x_{i+1}, \ldots,x_k \in \mathbb{R}_{\geq 0}$. By the above argument, we will get a valid solution for only one such $i$, which we return as the unique solution. Clearly it takes $k^{O(1)}$ time.

\subsection{Proof of Lemma \ref{lemma:marginalization}}
\label{proof:lemma-marginalization}
\begin{enumerate}
    \item Since $V_1\prec\ldots\prec V_{N+1}$ is a topological order, marginalizing over $V_{N+1}$ in the factorization $\mathbb{P}(\bm{V}) = \prod_{i=1}^{N+1}\mathbb{P}(V_i | \bm{pa}(V_i))$ will give the factorization $\mathbb{P}(\bm{V}_N) = \prod_{i=1}^{N}\mathbb{P}(V_i | \bm{pa}(V_i))$ which is the factorization over $\mathcal{G}_N$.
    \item Let $\mathcal{T} = \{(\pi_1, \bm{t}_1), \ldots , (\pi_m, \bm{t}_m)\}$ be a set of intervention tuples generating $\mathbb{P}_{mix}(\bm{V})$ and $C_{V_{N+1}} = \{v^1, \ldots, v^k\}$. Marginalizing with respect to $V_{N+1}$ for a single target $\bm{t}_i, i\in [m]$ looks like,
        \begin{equation}
        \label{equation:marginal-identity}
         \sum\limits_{l=1}^k \mathbb{P}_{\bm{t}_i}(\bm{V}_N, v^l) = \left\{
             \begin{array}{@{}l@{\thickspace}l}
              \mathbb{P}_{\bm{t}_i}(\bm{V}_N)  &: V_{N+1}\notin T_i\\
              \mathbb{P}_{\bm{t}_i\setminus\{v^j\}}(\bm{V}_N) &: v^j \in \bm{t}_i
             \end{array}
          \right.   
        \end{equation}
        
        Applying this marginalization to Equation \ref{equation:input-mixture}, i.e., $\mathbb{P}_{mix}(\bm{V}) = \sum_{i=1}^m \pi_i \mathbb{P}_{\bm{t}_i}(\bm{V})$, gives a convex linear combination of different $\mathbb{P}_{\bm{s}}(\bm{V}_N)$, where $\bm{s}$ are values of some set of variables $\bm{S} \subset \bm{V}_N$, implying that $\mathbb{P}_{mix}(\bm{V}_N)$ is a mixture of interventions on $\mathcal{G}_N$.
        
    \item This is straight-forward. To query $\mathbb{P}(v_1, \ldots, v_N)$, we query $\mathbb{P}(v_1, \ldots, v_N, v_{N+1})$ for all $v_{N+1}\in C_{V_{N+1}}$ and sum them up. The same can be done to create access for $\mathbb{P}_{mix}(\bm{V}_N)$. Since we are summing at most $k_{max}$ terms, in $O(k_{max})$ time we can simulate access to both $\mathbb{P}(\bm{V}_N)$, $\mathbb{P}_{mix}(\bm{V}_N)$.
\end{enumerate}
\subsection{Proof of Lemma \ref{lemma:lifting}}
\label{proof:lemma-lifting}
This follows by the marginalization equation (Equation \ref{equation:marginal-identity} in Appendix \ref{proof:lemma-marginalization}).

\subsection{Proof of Lemma \ref{lemma:setting}}
\label{proof:lemma-setting}
 Let $\bm{s}\in \mathcal{S}_r$ where $r> i$. This means that $\bm{s}$ is either $\bm{s}_r$ or $\bm{s}_r\cup\{v\}$ for some $v\in C_{V_{N+1}}$. We show that 
 $\mathbb{P}_{\bm{s}_r}(\bm{v}_{i,l}) = \mathbb{P}_{\bm{s}_r}(\bm{s}_i\cup\bm{s}_{-i}\cup\{v^l\}) = 0$ for all $v^l\in C_{V_{N+1}}$. The proof for $\bm{s} = \bm{s}_r\cup\{v\}$ is identical. Since $i<r$, we get  that $\bm{s}_{r}\not\subseteq \bm{s}_{i}$. Now there are two cases, either the set of variables $\bm{S}_r \subseteq \bm{S}_i$ or $\bm{S}_r \not\subseteq \bm{S}_i$. In the first case, since $\bm{s}_{r}\not\subseteq \bm{s}_{i}$ we get that there is some variable $V_j \in \bm{S}_r, \bm{S}_i$ such that different values $v_j^r$ and $v_j^i$ belong to $\bm{s}_r$ and $\bm{s}_i$ respectively implying that $\mathbb{P}_{\bm{s}_r}(\bm{s}_i\cup\bm{s}_{-i}\cup\{v\})=0$. In the second case, there is some variable $V_j \in \bm{S}_r$ ($j\in [N]$) that is not in $\bm{S}_i$. Note that the ``missing value'' $\Bar{v}_j\in C_{V_j}$ (i.e. the one that is missing from all targets $\bm{s}_j, j\in [q]$) belongs to $\bm{s}_{-i}$ (since $V_j \notin \bm{S}_i$) but it cannot belong to $\bm{s}_r$ (since it is missing from all $\bm{s}_1, \ldots,\bm{s}_q$) $\Rightarrow \mathbb{P}_{\bm{s}_r}(\bm{s}_i\cup\bm{s}_{-i}\cup\{v\})=0$.

\subsection{Proof of Lemma \ref{lemma:scalar-relationship}}
\label{proof:lemma-scalar-relationship}
Let $\bm{s}\in \mathcal{S}_i$. Note that $\bm{s}$ is either $\bm{s}_i$ or $\bm{s}_i\cup\{v\}$ for some $v\in C_{V_{N+1}}$. Using Equation \ref{equation:marginal-identity} we get that for all 
$\bm{s}\in \mathcal{S}_i$, marginalization of $V_{N+1}$ in $\pi_{\bm{s}}\mathbb{P}_{\bm{s}}(\bm{V})$ gives $\pi_{\bm{s}}\mathbb{P}_{\bm{s}_i}(\bm{V}_N)$. Marginalizing $V_{N+1}$ in Equation \ref{equation:lifted-mixture}, converts the left hand side to $\mathbb{P}_{mix}(\bm{V}_N)$ and right-hand side to $\sum_{i=1}^q (\sum_{\bm{s}\in \mathcal{S}_i} \pi_{\bm{s}}) \mathbb{P}_{\bm{s}_i}(\bm{V}_N)$ giving a set of intervention tuples that generates $\mathbb{P}_{mix}(\bm{V}_N)$. By the inductive hypothesis, $\mathbb{P}_{mix}(\bm{V}_N)$ is generated by the unique set of intervention tuples $\mathcal{S}$ satisfying Assumption \ref{assm:intervention-exclusion}. Since the intervention targets in $\mathcal{S}$ and the ones in the set of intervention tuples we just obtained are the same, using uniqueness of $\mathcal{S}$ we get that $\mu_i = \sum\limits_{\bm{s}\in \mathcal{S}_i}\pi_{\bm{s}}$.

\subsection{Proof of Lemma \ref{lemma:si-choice}}
\label{proof:lemma-si-choice}
Note that since $V_{N+1}$ is the last node in the topological order, using the definition of interventions, we can conclude that,
\[
\mathbb{P}_{\bm{s}_i\cup\{v^{l^\prime}\}}
    (\bm{s}_{i}\cup\bm{s}_{-i} \cup \{v^l\}) = \mathbb{P}_{\bm{s}_i}
    (\bm{s}_{i}\cup\bm{s}_{-i})\delta_{v^{l^\prime}, v^l}
\]
where $v^l, v^{l^\prime} \in C_{V_{N+1}} = \{v^1, \ldots, v^k\}$. Recall that $\bm{v}_{i,l} = \bm{s}_i\cup \bm{s}_{-i}\cup \{v^l\}$. Now, on substituting for $\Delta$ using Equation \ref{equation:reduced-si} into Equation \ref{equation:further-reduced-si}, we obtain,
\[
\begin{split}
    & \mathbb{P}_{mix}(\bm{v}_{i,l}) 
    - \mu_i \mathbb{P}_{\bm{s}_i}(\bm{v}_{i,l})
    -\sum\limits_{j=1}^{i-1}\sum\limits_{\bm{s} \in \mathcal{S}_j} \pi_{\bm{s}}\mathbb{P}_{\bm{s}}
    (\bm{v}_{i,l}) =\\ 
    &\sum\limits_{l^{\prime}\in [k]}  \pi_{\bm{s}_i\cup\{v^{l^\prime}\}} \biggl(  \mathbb{P}_{\bm{s}_i}
    (\bm{s}_{i}\cup\bm{s}_{-i})\delta_{v^{l^\prime}, v^l}  
    - \mathbb{P}_{\bm{s}_i}
    (\bm{v}_{i,l})\biggr)
\end{split}
\]
Note that all the unknown variables are on the right-hand side of this equation. Varying $l^{\prime}\in [k]$, gives us a linear system of equations satisfied by scalars $\pi_{\bm{s}_i\cup \{v^{l^\prime}\}}$.

\begin{equation*}
\label{eq:t-base-sys}
\left.
\left.
    \begin{bmatrix}
        c - a_{1} & - a_{1} & . &. & - a_{1}\\
        - a_{2} & c - a_{2} & . & .& - a_{2}\\
        . & . & . &. &. \\
        - a_{k} & - a_{k} & . & . & c - a_{k}\\
    \end{bmatrix}
\right.
    \begin{bmatrix}
        x_1\\
        x_2\\
        .\\
        x_k\\
    \end{bmatrix}
\right.=
    \begin{bmatrix}
    b_1\\
    b_2\\
    .\\
    b_k\\
    \end{bmatrix}
\end{equation*}
In the above system, we have renamed the known values as follows. For $l\in [k]$, denote 
\begin{equation*}
    \begin{split}
  & a_l = \mathbb{P}_{\bm{s}_i}(\bm{v}_{i,l}),\\
  & b_l= \mathbb{P}_{mix}(\bm{v}_{i,l}) 
    - \mu_i \mathbb{P}_{\bm{s}_i}(\bm{v}_{i,l})
    -\sum\limits_{j=1}^{i-1}\sum\limits_{\bm{s} \in \mathcal{S}_j} \pi_{\bm{s}}\mathbb{P}_{\bm{s}}
    (\bm{v}_{i,l}),\\
  & c = \mathbb{P}_{\bm{s}_i}(\bm{s}_i\cup\bm{s}_{-i})
\end{split}
\end{equation*}

All $a_l$'s are probabilities from interventional distributions and can be computed as product of conditional probabilities. Thus, by Assumption \ref{assm:positivity}, $a_l >0$ for all $l\in [k]$. It's easy to see that $c=\sum_{l\in [k]}a_l$, by the sum rule of probability. By statement of Lemma \ref{lemma:lifting},  for each $i\in [q]$ and $l^\prime\in [k]$, $\pi_{\bm{s}_i\cup\{v^{l^\prime}\}} \geq 0$. Since we are only considering set of intervention tuples  which satisfy Assumption \ref{assm:intervention-exclusion}, there is some $l^\prime\in [k]$ such that $\pi_{\bm{s}_i\cup\{v^{l^\prime}\}}=0$. On constraining the variables $x_1,\ldots,x_k$ in the above system to these conditions (i.e. $x_{l^\prime}\geq 0$ for all $l^\prime\in [k]$ and $x_{l^\prime}=0$ for some $l^\prime\in [k]$), by Lemma \ref{lemma:sys-eqn} we are guaranteed a unique solution. Therefore there is a unique tuple $(\pi_{\bm{s}_i\cup\{v^1\}}, \ldots,\pi_{\bm{s}_i\cup\{v^k\}})$ satisfying these requirements $\Rightarrow$ Equation \ref{equation:further-reduced-si} has a unique solution which is easily computed in $k^{O(1)}$ time using the technique described in proof of Lemma \ref{lemma:sys-eqn}.

\section{Non-Necessity of Assumption \ref{assm:positivity}}
\label{appendix:assumption-non-necessity}
With the help of an example we argue that Assumption \ref{assm:positivity} is not necessary in Theorem \ref{theorem:general-identify}.

\begin{example}
\label{eg:positivity-counterexample}
Consider a causal Bayesian Network 
\[
V_1 \rightarrow V_2
\]
defined over two binary variable $\bm{V}=\{V_1,V_2\}$ with $C_{V_i}=\{0,1\}$, $i\in [2]$. Further, define CPDs,
$\mathbb{P}(V_1 = 1)=0.5$, $\mathbb{P}(V_2=1|V_1=0)=0.5$, and $\mathbb{P}(V_2=1|V_1=1)=0$. Clearly $\mathbb{P}(V_2=1,V_1=1)=0$ implying that this CBN doesn't satisfy Assumption \ref{assm:positivity}.

Let $\mathbb{P}_{mix}(\bm{V})$ be a mixture distribution defined as 
\[
\frac{1}{2} \mathbb{P}(\bm{V}|do(V_1=0)) + \frac{1}{2} \mathbb{P}(\bm{V} | do(V_1=0, V_2=0))
\]
 
 This mixture satisfies Assumption \ref{assm:intervention-exclusion}. Our algorithm first marginalizes on $V_2$ and tries to find the unique set of intervention targets for $\mathbb{P}_{mix}(V_1)$. For this sub-problem, all steps of Algorithm \ref{algo:general-algo-inf} go through (since the distribution $\mathbb{P}(V_1)$ satisfies positivity), and the correct components get identified. Note that, for this sub-problem the algorithm identifies that $\mathbb{P}_{mix}(V_1) = \mathbb{P}(V_1|do(V_1=0))$.
 Now it tries to lift this computed target $(V_1=0)$ to targets for the full mixture $\mathbb{P}_{mix}(\bm{V})$.
 
 Since the algorithm does not try to lift the target $(V_1=1)$ (as it was not found as a target for $\mathbb{P}_{mix}(V_1)$), it does not require $\mathbb{P}(V_2 |V_1=1)$ to be non-zero. This can be easily checked in the lifting process described in Section \ref{subsubsection:lifting}. We do not repeat the steps of our algorithm here and encourage the reader to work through the lifting steps outlined in Section \ref{subsubsection:lifting} and obtain unique solutions proving our point that Assumption \ref{assm:positivity} is not necessary and can be weakened.
  \end{example}

\section{Worked-Out Examples}
\label{appendix:workout-example}
In this section, we illustrate the workings of Algorithm \ref{algo:general-algo-inf} (in the main paper) using two worked-out examples. Example \ref{eg:workout-example} is simpler and uses a mixture distribution on a CBN with just two nodes. It does not really require all of the crucial ideas from the lifting procedure described in Section \ref{subsubsection:lifting}. However, we believe it is important since it gives a good broad understanding of the entire algorithm.  Example \ref{eg:workout-example-2} is complicated enough (using a mixture distribution on a CBN with three nodes) to highlight some of the key novelties of our lifting procedure in Section \ref{subsubsection:lifting}. We urge the reader to first work through Example \ref{eg:workout-example} and then through Example \ref{eg:workout-example-2} to get a full understanding of the critical ideas that make our proof of Theorem \ref{theorem:general-identify} work.

\begin{example}
\label{eg:workout-example}
Consider a causal Bayesian Network 
\[
V_1 \rightarrow V_2
\]
defined over two binary variable $\bm{V}=\{V_1,V_2\}$ with $C_{V_i}=\{0,1\}$, $i\in [2]$. Further, define CPDs,
$\mathbb{P}(V_1 = 1)=0.5$, $\mathbb{P}(V_2=1|V_1)=0.5$ for both $V_1=0$ and $V_1=1$. Clearly, this CBN satisfies Assumption \ref{assm:positivity}. Let $\mathbb{P}_{mix}(\bm{V})$ be a mixture distribution defined as 
\[
\frac{1}{2} \mathbb{P}(\bm{V}|do(V_1=0)) + \frac{1}{2} \mathbb{P}(\bm{V} | do(V_1=0, V_2=0))
\]
 
 This mixture satisfies Assumption \ref{assm:intervention-exclusion}. On marginalizing variable $V_2$, the mixture becomes $\mathbb{P}_{mix}(V_1) = \mathbb{P}(V_1|do(V_1=0))$ which also satisfies Assumption \ref{assm:intervention-exclusion}, as a mixture of interventions on the CBN comprising of the single variable $V_1$. A general mixture distribution on variable $V_1$ looks like,
 \begin{equation*}
  \begin{split}
     \pi_0 \mathbb{P}(V_1|do(V_1=0)) + \pi_1 \mathbb{P}(V_1|do(V_1=0)) + (1-\pi_0-\pi_1)\mathbb{P}(V_1)
 \end{split}   
 \end{equation*}

Varying $V_1$ in $\{0,1\}$ gives the system of equations,
\begin{equation*}
   \begin{bmatrix}
        \begin{array}{cc}
             1-0.5 & -0.5 \\
            -0.5 & 1-0.5 \\
        \end{array}
    \end{bmatrix}
    \begin{bmatrix*}[l]
        \pi_0 \\ \pi_1
    \end{bmatrix*}
        =
    \begin{bmatrix*}[l]
        \begin{array}{cc}
             & \mathbb{P}_{mix}(V_1=0)-0.5 \\
             & \mathbb{P}_{mix}(V_1=1)-0.5 \\
        \end{array}
    \end{bmatrix*} 
\end{equation*}

Lemma \ref{lemma:sys-eqn} highlights why such a system has a unique non-negative solution if we assume that at least one of $\pi_0, \pi_1$ is $0$, i.e. the set of intervention tuples we are trying to construct satisfy Assumption \ref{assm:intervention-exclusion}. Lemma \ref{lemma:sys-eqn} also provides an efficient algorithm to find the unique solution giving $\mathbb{P}_{mix}(V_1) = \mathbb{P}(V_1|do(V_1=0))$. Now we have one intervention target $(V_1=0)$ at hand, which we will try to lift. Note that, the possible lifts of such a target are $(V_1=0), (V_1=0, V_2=0), (V_1=0, V_2=1)$. Our next step will search within this space of targets and try to complete the construction. A general solution for mixtures within this space would look like
 \begin{equation*}
  \begin{split}
     \mu_0 \mathbb{P}(\bm{V}|do(V_1=0)) + \mu_1 \mathbb{P}(\bm{V}|do(V_1=0, V_2=0)) + \mu_2\mathbb{P}(\bm{V} | do(V_1=0, V_2=1))
 \end{split}   
 \end{equation*}
 where $\mu_i\geq 0, i\in \{0,1,2\}$. So we want to find $\mu_0, \mu_1,\mu_2$ such that the above general solution becomes equal to $\mathbb{P}_{mix}(\bm{V})$. Since we already know that $\mathbb{P}_{mix}(V_1) = \mathbb{P}(V_1|do(V_1=0))$, marginalizing on $V_2$ implies that $\mu_0+\mu_1+\mu_2=1$. We substitute $\mu_0 = 1-\mu_1-\mu_2$, equate the above general mixture to $\mathbb{P}_{mix}(\bm{V})$, and re-arrange terms to get,
  \begin{equation*}
  \begin{split}
  \mathbb{P}_{mix}(\bm{V}) - 
     \mathbb{P}(\bm{V}|do(V_1=0)) &=\mu_1 (\mathbb{P}(\bm{V}|do(V_1=0, V_2=0))- \mathbb{P}(\bm{V}|do(V_1=0))) \\&+ \mu_2(\mathbb{P}(\bm{V} | do(V_1=0, V_2=1))-\mathbb{P}(\bm{V}|do(V_1=0)))
 \end{split}   
 \end{equation*}
 
 We evaluate this mixture at settings $V_1=0,V_2=0$, and $V_1=0,V_2=1$, to get a system of linear equations in $(\mu_1, \mu_2)$,
\begin{equation*}
   \begin{bmatrix}
        \begin{array}{cc}
             1-0.5 & -0.5 \\
            -0.5 & 1-0.5 \\
        \end{array}
    \end{bmatrix}
    \begin{bmatrix*}[l]
        \mu_1 \\ \mu_2
    \end{bmatrix*}
        =
    \begin{bmatrix*}[l]
        \begin{array}{cc}
             & \mathbb{P}_{mix}(0,0)-0.5 \\
             & \mathbb{P}_{mix}(0,1)-0.5 \\
        \end{array}
    \end{bmatrix*} 
\end{equation*}
  
Again, since we are solving for a set of intervention tuples satisfying Assumption \ref{assm:intervention-exclusion}, one of $\mu_1, \mu_2$ would be $0$. This used with Lemma \ref{lemma:sys-eqn} gives the unique solution $\mu_1=0.5, \pi_2=0 \Rightarrow \mu_0=0.5$, thereby identifying the correct set of intervention tuples.

\end{example}  
We would like to note that the example we illustrated above is rather simple and does not capture some main non-trivial aspects of our Algorithm. However, we think it is important as a warm up exercise. A more involved example that would bring out the crucial ideas of our proof is presented below in Example \ref{eg:workout-example-2}.
\begin{example}
\label{eg:workout-example-2}
Consider a causal Bayesian Network defined over three binary variables $\bm{V}=\{V_1,V_2, V_3\}$ taking values in $C_{V_i}=\{0,1\}$, $i\in [3]$, with
$\bm{pa}(V_1)=\varnothing$, $\bm{pa}(V_2)=\{V_1\}$ and $\bm{pa}(V_3) = \{V_1,V_2\}$. Let CPDs be such that Assumption \ref{assm:positivity} is satisfied. Consider a mixture distribution 
\begin{equation*}
    \begin{split}
        \mathbb{P}_{mix}(\bm{V}) = \mu_0 \mathbb{P}(\bm{V}|do(V_1=0)) + \mu_1 \mathbb{P}(\bm{V} | do(V_1=0, V_2=0)) +\mu_2 \mathbb{P}(\bm{V} | do(V_1=0, V_2=0, V_3=0))
    \end{split}
\end{equation*}
with positive scalars $\mu_0,\mu_1,\mu_2$ satisfying $\mu_0+\mu_1+\mu_2=1$. Marginalizing on $V_3$, gives,
\begin{equation*}
    \begin{split}
        \mathbb{P}_{mix}(V_1,V_2) = \pi_0 \mathbb{P}(V_1, V_2|do(V_1=0)) + \pi_1 \mathbb{P}(V_1, V_2 | do(V_1=0, V_2=0))
    \end{split}
\end{equation*}
Our inductive hypothesis assumes that this mixture on smaller number of nodes is generated by a unique set of intervention tuples that satisfies Assumption \ref{assm:intervention-exclusion} and also that this set can be efficiently computed. This will give us access to the two scalars $\pi_0, \pi_1$ and to the two targets
 $(V_1=0)$ and $(V_1=0, V_2=0)$ (we call them the currently computed targets). We need to lift these targets to targets for the original mixture distribution like we did in Example \ref{eg:workout-example}. However, the situation is not as simple here. Note that $(V_1=0)$ can be lifted to one of the three targets $(V_1=0), (V_1=0, V_3=0), (V_1=0, V_3=1)$. Similarly $(V_1=0, V_2=0)$ can be lifted to one of the three targets $(V_1=0, V_2=0), (V_1=0, V_2=0, V_3=0), (V_1=0, V_2=0, V_3=1)$. So there are $6$ possible targets in the original mixture and therefore a general solution for our mixture can be written using $6$ new variables (say $\delta_0, \delta_1, \delta_2, \delta_3, \delta_4, \delta_5$) such that,
\begin{equation}
\label{equation:master}
    \begin{split}
        \mathbb{P}_{mix}(\bm{V}) &= \delta_0 \mathbb{P}(\bm{V}|do(V_1=0)) + \delta_1 \mathbb{P}(\bm{V} | do(V_1=0, V_3=0)) +\delta_2 \mathbb{P}(\bm{V} | do(V_1=0, V_3=1)) \\&+ \delta_3 \mathbb{P}(\bm{V}|do(V_1=0, V_2=0)) + \delta_4 \mathbb{P}(\bm{V} | do(V_1=0, V_2=0, V_3=0))\\ &+\delta_5 \mathbb{P}(\bm{V} | do(V_1=0, V_2=0, V_3=1))
    \end{split}
\end{equation}
where all $\delta_i$ are non-negative. By marginalizing on $V_3$, and using the solution we got from the inductive hypothesis (like we did in Example \ref{eg:workout-example}), we can show that,
\[
\pi_0 = \delta_0+\delta_1 + \delta_2,\hspace{1em}\pi_1 =  \delta_3 + \delta_4 + \delta_5
\]
Now the two main non trivial ingredients needed from here are:
\begin{itemize}
    \item Deciding the order in which the currently computed targets should be lifted, and
    \item Deciding the settings for $\bm{V}$ that would give linear systems where we can argue about unique solutions like in Example \ref{eg:workout-example}.
\end{itemize}

For the first one, we lift the currently computed targets in an order which does not violate set inclusion for these targets, i.e. we first lift $(V_1=0)$ and then lift $(V_1=0, V_2=0)$. This can be done by considering any extension of the set inclusion partial order on these targets. Then for lifting the target $(V_1=0)$, we choose to evaluate on the settings $\bm{v}_1 = (V_1=0, V_2=1, V_3=0), \bm{v}_2= (V_1=0, V_2=1, V_3=1)$. Here, we pick the value of $V_2$ (i.e. $V_2=1$) that is missing from the currently computed target under consideration i.e. $(V_1=0)$. There will always be one such missing value (it follows from Assumption \ref{assm:intervention-exclusion}). Evaluating on these settings simplifies our equation drastically. For $l\in [2]$, we get,
\begin{equation*}
    \begin{split}
        \mathbb{P}_{mix}(\bm{v}_l) = \delta_0 \mathbb{P}(\bm{v}_l|do(V_1=0)) + \delta_1 \mathbb{P}(\bm{v}_l | do(V_1=0, V_3=0)) +\delta_2 \mathbb{P}(\bm{v}_l | do(V_1=0, V_3=1))
    \end{split}
\end{equation*}
Basically all possible lifts of the other currently computed target i.e. $(V_1=0, V_2=0)$ vanish and we have a much simpler system of equations at hand. From here the solution follows exactly like the previous example. We substitute $\delta_0 = \pi_0-\delta_1-\delta_2$ and rearrange to get a linear system in $2$ equations and $2$ variables $\delta_1, \delta_2$. Similar to the argument made in Example \ref{eg:workout-example}, at least one of $\delta_1, \delta_2$ will be $0$ and therefore this system has a unique solution (using Lemma \ref{lemma:sys-eqn}) giving values of $\delta_0, \delta_1, \delta_2$. These can then be substituted back in Equation \ref{equation:master} reducing the number of variables to $3$ (i.e. $\delta_3, \delta_4, \delta_5$). Again we substitute $\delta_3 = \pi_1 - \delta_4-\delta_5$ and reduce the equation to just two unknowns. Finally by  using settings $\bm{v}_1 = (V_1=0, V_2=0, V_3=0), \bm{v}_2= (V_1=0, V_2=0, V_3=1)$ in Equation \ref{equation:master} we will be left with $2$ equations in $2$ variables. Exactly like our argument for lifting of $(V_1=0)$ (i.e. using Lemma \ref{lemma:sys-eqn}), we can show that, when at least one of $\delta_4, \delta_5$ is $0$, this system has a unique solution as well. Lemma \ref{lemma:sys-eqn} would efficiently give us the values of $\delta_4, \delta_5$ and therefore of $\delta_3$. Thus, we uniquely identify a set of intervention tuples that satisfies Assumption \ref{assm:intervention-exclusion} and generates $\mathbb{P}_{mix}(\bm{V})$.
\end{example}
\section{Finite Sample Algorithm}
\label{appendix:general-algo}

In a real world scenario, we will only have finitely many samples from the distributions $\mathbb{P}(\bm{V})$ and $\mathbb{P}_{mix}(\bm{V})$. In this situation, we modify Algorithm \ref{algo:general-algo-inf}  slightly to make it work with finitely many samples. The resulting algorithm is presented in Algorithm \ref{algo:general-algo}. Let the sets containing the samples be $\mathcal{B}=\{\bm{b}_1,\ldots,\bm{b}_{M}\}$ where $\bm{b}_i \sim \mathbb{P}(\bm{V})$ and $\mathcal{B}_{mix}=\{\bm{b}_1^{mix},\ldots,\bm{b}_{M}^{mix}\}$  where $\bm{b}_j^{mix} \sim \mathbb{P}_{mix}(\bm{V})$. 
As a preprocessing step, we estimate the distributions $\mathbb{P}(\bm{V})$ and $\mathbb{P}_{mix}(\bm{V})$ as $\hat{\mathbb{P}}(\bm{V})$ and $\hat{\mathbb{P}}_{mix}(\bm{V})$ (respectively) using samples in $\mathcal{B}$ and $\mathcal{B}_{mix}$ respectively. $\hat{\mathbb{P}}(\bm{V})$ is estimated by estimating all the CPDs using maximum likelihood estimation (MLE). In our implementation, we use a function from the pgmpy library (\cite{Ankan2015}), to compute these MLE estimates. We enforce Assumption \ref{assm:positivity} on $\hat{\mathbb{P}}(\bm{V})$, by enforcing it on all it's CPDs, using a small positive parameter $\delta$ (chosen by us). When $\hat{\mathbb{P}}(v_i | \bm{pa}(v_i)) = 0$ for some $v_i$ and some setting of parents $\bm{pa}(v_i)$, we update,
\[
\hat{\mathbb{P}}(V_i | \bm{pa}(v_i)) \gets \hat{\mathbb{P}}(V_i | \bm{pa}(v_i)) + \delta
\]
for all values of $V_i$, and then re-normalize to make it a probability distribution again. Marginal $\hat{\mathbb{P}}_{mix}(\bm{V} = \bm{v})$ is calculated using relative frequency of the occurence of $\bm{V} = \bm{v}$ in the samples inside $\mathcal{B}_{mix}$. These estimated distributions are then used as inputs in Algorithm \ref{algo:general-algo}. We use another small positive parameter $\epsilon$ as input to Algorithm \ref{algo:general-algo} which prunes each recovered set of intervention tuples computed during the algorithm, by only keeping mixing coefficients greater that $\epsilon$.
It's easy to see that the time complexity of Algorithm \ref{algo:general-algo}, including the estimation of probabilities from samples is:
\begin{equation*}
    \biggr(\frac{Nk_{max}^dM}{\epsilon}\biggr)^{O(1)}
\end{equation*}
Here $N$ is number of nodes in the CBN $\mathcal{G}$, $d$ is the maximum in-degree of any node in $\mathcal{G}$, $k_{max}$ is maximum number of values that any node in $\mathcal{G}$ can take and $M$ is the number of samples present in $\mathcal{B}$ and $\mathcal{B}_{mix}$.

\begin{remark}
Since our algorithm's run-time depends on $\epsilon$, we need to carefully select it's value. Setting it too small could increase the run time whereas setting it too big could lead to wrongfully pruning intervention targets (with significant mixing proportions) present in the mixture.
\end{remark}

\begin{algorithm}[!ht]
\caption{DISENTANGLE-FINITE}
\label{algo:general-algo}
\SetAlgoLined
\DontPrintSemicolon
\SetKwInOut{Input}{input}
\SetKwInOut{Output}{output}
\SetKw{Return}{return}

\Input{$\bm{V}, \mathcal{G} \hat{\mathbb{P}}(\bm{V}), \hat{\mathbb{P}}_{mix}(\bm{V}), \epsilon$}

\Output{Set of intervention tuples $\mathcal{T}$}
\BlankLine

\begin{enumerate}

    \item When $|\bm{V}|=1$, setup the linear system in Equation \ref{eq:t-basecase-sys} (say $\bm{Ax}=\bm{b}$) using the estimated distributions. Similar to the technique described in Lemma \ref{lemma:sys-eqn}, set one variable to $0$ at a time giving solution $(\pi_1, \ldots, \pi_k)$ corresponding to targets $(\bm{t}_1,\ldots,\bm{t}_k)$ as described in Section \ref{subsection:base-case}. For every variable that is set to $0$, create a set $\mathcal{T} = \{(\bm{t}_i, \pi_i):i\in [k]\}$ containing the solution. For every such $\mathcal{T}$, iterate through the tuples $(\bm{t}_i, \pi_i)$ in it. If some $\pi_i < 0$, set $\pi_i\gets 0$. Compute the score $r(\mathcal{T}) = \|\bm{A\pi}-\bm{b}\|^2$, where $\bm{\pi} = (\pi_1, \ldots, \pi_k)$. Next, select $\mathcal{T}$ with the smallest value of $r(\mathcal{T})$. For this selected $\mathcal{T}$, check if $1-\sum_{i=1}^k\pi_i < \epsilon$. If yes, renormalize $\pi_i \gets \pi_i / (\sum_{i=1}^k\pi_i)$. If no, add the tuple  $(\varnothing, 1-\sum_{i=1}^k \pi_i)$ to $\mathcal{T}$. Only keep the tuples with strictly positive mixing coefficients i.e. $\mathcal{T}\gets \{(\bm{t}_i, \pi_i) \in \mathcal{T} : \pi_i > 0\}$. \Return $\mathcal{T}$.
    
    \item Let $V_1\prec\ldots\prec V_{N+1}$ denote a topological order in $\mathcal{G}$. Marginalize on $V_{N+1}$ to create access to $\hat{\mathbb{P}}_{mix}(\bm{V}_N)$ and $\hat{\mathbb{P}}(\bm{V}_N)$ where $\bm{V}_N = (V_1, \ldots,V_N)$. Construct $\mathcal{G}_N = \mathcal{G}\setminus \{V_{N+1}\}$. Recursively call this algorithm with inputs $\mathcal{G}_N$, $\hat{\mathbb{P}}(\bm{V}_N)$, $\hat{\mathbb{P}}_{mix}(\bm{V}_N)$, and obtain a set of intervention tuples $\mathcal{S} = \{(\bm{s}_1, \mu_1), \ldots, (\bm{s}_q, \mu_q)\}$. Let $\bm{s}_1, \ldots, \bm{s}_q$ be ordered such that $i\leq j$ implies that $\bm{s}_j \not\subseteq \bm{s}_i$. For all $i\in [N]$, by inspecting $\bm{s}_j$, identify $\Bar{v}_i\in C_{V_i}$ such that $\Bar{v}_i\notin \bm{s}_j$ for any $j\in [q]$.
    Define $\bm{s}_{-j} = \{\Bar{v}_i : V_i\notin S_j\}$. Let $C_{V_{N+1}} = \{v^1, \ldots, v^k\}$. For each $i\in[q]$ and $l\in [k]$, create setting $\bm{v}_{i,l} = \bm{s}_i\cup \bm{s}_{-i}\cup\{v^l\}$. 
    
    \item For each fixed $i\in [q]$, evaluate distributions for different $\bm{v}_{i,l}$, $l\in [k]$, to setup the system of equations (say $\bm{Ax}=\bm{b}$) described in Equation \ref{equation:further-reduced-si}. Similar to the technique described in Lemma \ref{lemma:si-choice} (which in turn uses Lemma \ref{lemma:sys-eqn}), set one variable to $0$ at a time giving solution $(\pi_{\bm{s}_i\cup\{v^1\}}, \ldots, \pi_{\bm{s}_i\cup\{v^k\}})$ corresponding to targets $(\bm{s}_i\cup\{v^1\},\ldots,\bm{s}_i\cup\{v^k\})$ as described in Section \ref{subsection:induction-step}. For every variable that is set to $0$, create a set $\mathcal{T} = \{(\bm{s}_i\cup\{v^{l^\prime}\}, \pi_{\bm{s}_i\cup\{v^{l^\prime}\}}):{l^\prime}\in [k]\}$ containing the solution. For every such $\mathcal{T}$, iterate through the tuples in it. If some $\pi_{\bm{s}_i\cup\{v^{l^\prime}\}} < 0$, set $\pi_{\bm{s}_i\cup\{v^{l^\prime}\}}\gets 0$. Compute the score $r(\mathcal{T}) = \|\bm{A\pi}-\bm{b}\|^2$, where $\bm{\pi} = (\pi_{\bm{s}_i\cup\{v^1\}}, \ldots, \pi_{\bm{s}_i\cup\{v^k\}})$. Next, select $\mathcal{T}$ with the smallest value of $r(\mathcal{T})$. For this selected $\mathcal{T}$, check if $\mu_i-\sum_{{l^\prime}=1}^k\pi_{\bm{s}_i\cup\{v^{l^\prime}\}} < \epsilon$. If yes, re-normalize $\pi_{\bm{s}_i\cup\{v^{l^\prime}\}} \gets (\mu_i \times \pi_{\bm{s}_i\cup\{v^{l^\prime}\}}) / (\sum_{{l^{\prime\prime}}=1}^k\pi_{\bm{s}_i\cup\{v^{l^{\prime\prime}}\}})$. If no, add the tuple  $(\bm{s}_i, \mu_i-\sum_{{l^\prime}=1}^k \pi_{\bm{s}_i\cup\{v^{l^\prime}\}})$ to $\mathcal{T}$. At the end of this process, collect all the intervention tuples thus obtained (for all $i\in [q]$), in the set $\mathcal{T}$.
    
    \item Find the excluded value of node $V_{N+1}$, i.e. the value which is not present in any target in $\mathcal{T}$. If no such value exists, find $v\in C_{V_{N+1}}$ which minimizes $\sum_{i=1}^q \pi_{\bm{s}_i\cup\{v\}}$. For each $i\in [q]$, set $\pi_{\bm{s}_i\cup\{v\}}\gets 0$. For each $i\in [q]$, renormalize the mixing coefficients $\pi_{\bm{s}_i\cup\{v^{l^\prime}\}} \gets (\pi_{\bm{s}_i\cup\{v^{l^\prime}\}}\times \mu_i)/(\sum_{{l^{\prime\prime}}=1}^ k\pi_{\bm{s}_i\cup\{v^{l^{\prime\prime}}\}})$. Only keep the tuples with strictly positive mixing coefficients in $\mathcal{T}$ i.e. $\mathcal{T}\gets \{(\bm{s}, \pi_{\bm{s}})\in \mathcal{T} : \pi_{\bm{s}} > 0\}$. \Return $\mathcal{T}$
\end{enumerate}
\end{algorithm}

\section{Additional Simulations}
\label{appendix:experiments}
\subsection{Effect of Graph Size}
\label{subsec:graph_size}

\begin{figure*}[ht!]
\centering
\includegraphics[width=\textwidth,height=0.45\textwidth]{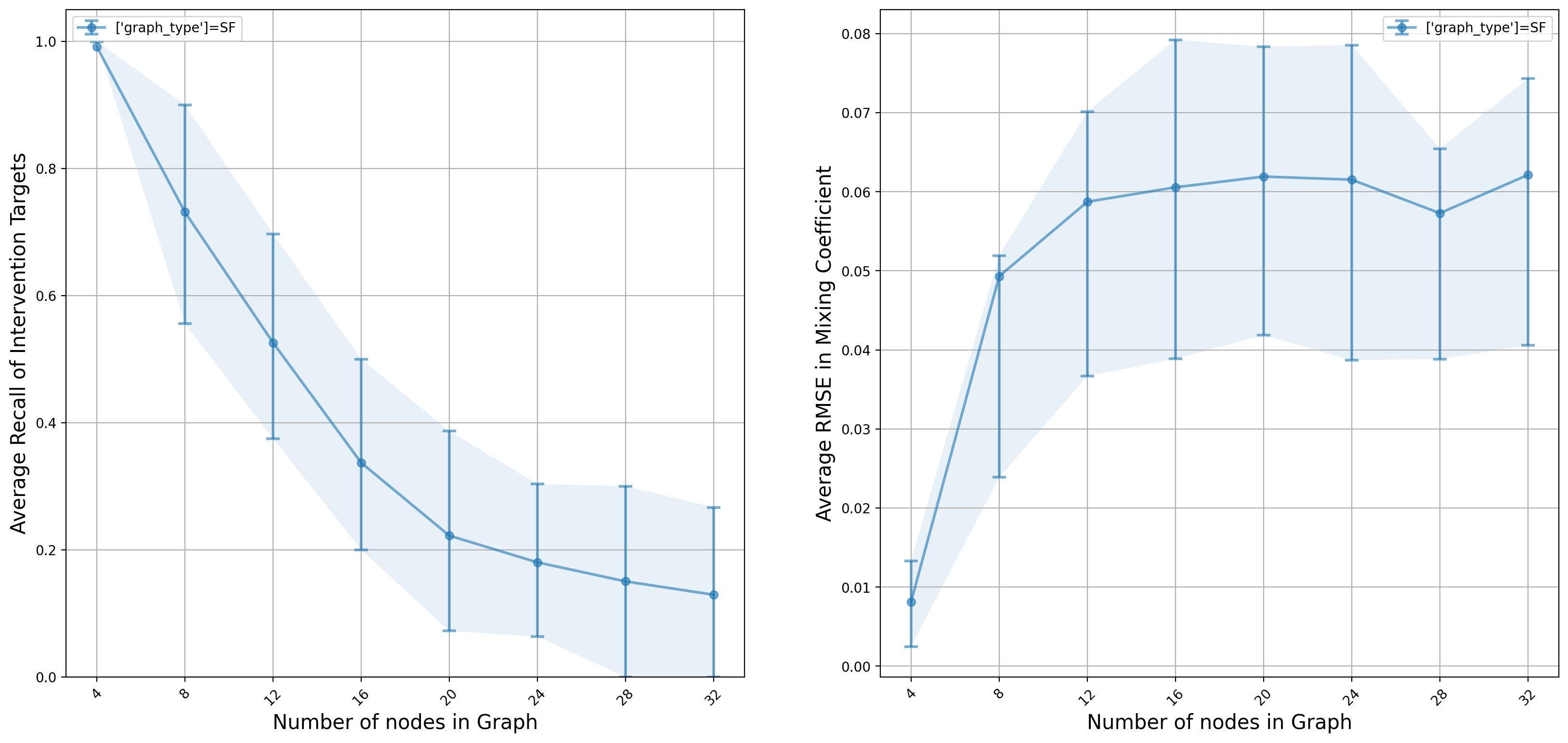}
\caption{Performance of Algorithm \ref{algo:general-algo} as a function of number of nodes in the graph. The error bars show the $20\%$ and $80\%$ quantile respectively.}
\label{fig:node-variation}
\end{figure*}

Figure \ref{fig:node-variation} shows the variation of performance of Algorithm \ref{algo:general-algo} keeping the number of samples fixed at $\sim 10^6$. We observe that recall decreases and root-mean-squared error in mixing coefficient increases very quickly as the number of nodes increases in the graph. Even though this is expected since error is accumulated as we successively add nodes and find new intervention targets, such performance for a very large sample size indicates bad dependence of sample complexity on the number of nodes. Improving this needs more exploration and is left for future work.

\subsection{Effect of Graph Type}
\label{subsec:graph_type}

\begin{figure*}[ht!]
\centering
\includegraphics[width=\textwidth,height=0.45\textwidth]{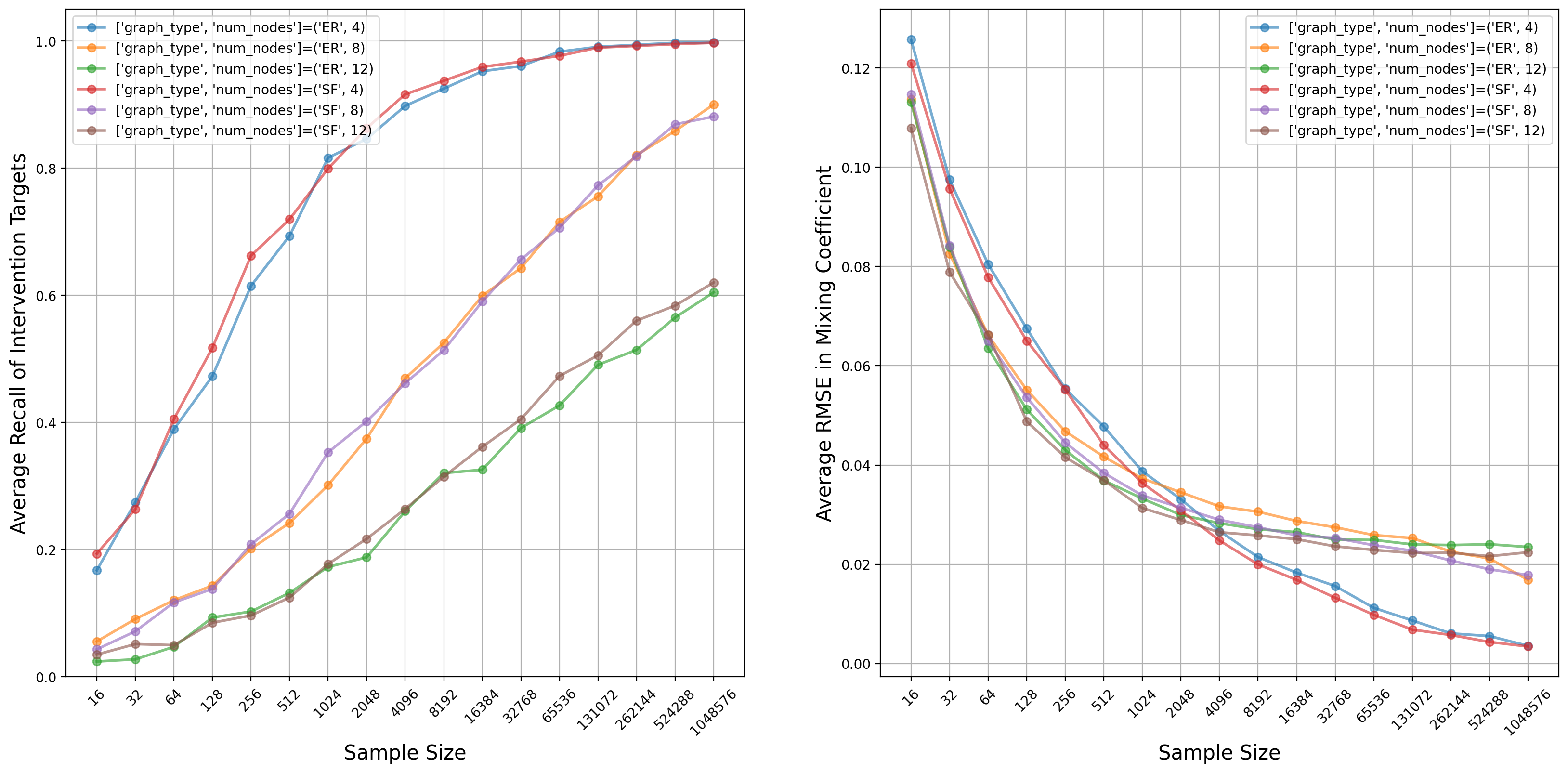}
\caption{Comparison of performance of Algorithm \ref{algo:general-algo} for CBNs generated from Erd\"os-R\'enyi (ER) and Scale Free (SF) models.}
\label{fig:graph-comparison}
\end{figure*}

In Figure \ref{fig:graph-comparison}, we demonstrate performance of Algorithm \ref{algo:general-algo} for CBNs generated from two different family of random graphs (Erd\"os-R\'enyi (ER) and Scale-Free (SF)). We observe no significant difference in performance for these models and make a conjecture that only high level graph parameters (such as number of nodes, edges, in-degree etc.) might be having an impact on performance and the topology (given these parameters) might not be that crucial.

\section{Evaluation Metrics}
\label{appendix:exp-metrics}
Let $\mathcal{T}$ denote the actual set of intervention targets and $\widehat{\mathcal{T}}$ denote the set of intervention targets computed by our algorithm. Let $\pi_{\bm{t}}$, $\widehat{\pi}_{\bm{s}}$ denote mixing coefficients of target $\bm{t}$, $\bm{s}$ in $\mathcal{T}$ and $\widehat{\mathcal{T}}$ respectively. We use the following evaluation metrics to evaluate the performance of our algorithm.
\begin{enumerate}
    
    \item \textbf{Recall}: Proportion of number of targets in $\mathcal{T}$ that were correctly identified in  $\mathcal{\widehat{T}}$ 
    \begin{equation*}
        \text{Recall} = \frac{|\mathcal{T}\cap\mathcal{\widehat{T}}|}{|\mathcal{T}|}.
    \end{equation*}
    
    \item \textbf{Root Mean Squared Error}: Root-mean-squared error (RMSE) in the mixing coefficients.
    \begin{equation*}
        \text{RMSE} = \sqrt{\frac{
            \begin{aligned}
                 \sum_{\bm{t}\in \mathcal{T}\cap \mathcal{\widehat{T}}}(\pi_{\bm{t}}-\widehat{\pi}_{\bm{t}})^2 + \sum_{\bm{t} \in (\mathcal{T}\setminus \mathcal{\widehat{T}})} (\pi_{\bm{t}})^2
                 +\sum_{\bm{t}\in (\mathcal{\widehat{T}}\setminus \mathcal{T})} (\widehat{\pi}_{\bm{t}})^2 \\
            \end{aligned}
        }{
            |\mathcal{T} \cup \mathcal{\widehat{T}}|}
        }
    \end{equation*}
    
    \item \textbf{False-Positive RMSE}: RMSE in the mixing coefficients of the incorrectly identified targets.
    \begin{equation*}
        \text{FP-RMSE} = \sqrt{\frac{\sum_{\bm{t} \in (\mathcal{\widehat{T}}\setminus \mathcal{T})} (\widehat{\pi}_{\bm{t}})^2}{|\mathcal{\widehat{T}} \setminus \mathcal{T}|}}
    \end{equation*}
    
    \item \textbf{False-Negative RMSE}: RMSE in the mixing coefficients of targets not identified.
    \begin{equation*}
        \text{FN-RMSE} = \sqrt{\frac{\sum_{\bm{t} \in (\mathcal{T}\setminus \mathcal{\widehat{T}})} (\pi_{\bm{t}})^2}{|\mathcal{T} \setminus \mathcal{\widehat{T}}|}}
    \end{equation*}
\end{enumerate}

\end{document}